\documentclass[lettersize,journal]{IEEEtran}
\usepackage{pgfplots}
\pgfplotsset{compat=1.16}
\usepackage{amssymb}
\usepackage{amsmath}
\usepackage{algorithm}
\usepackage{algpseudocode}
\renewcommand{\Comment}[1]{\textit{\# #1}} 
\usepackage{mathrsfs}
\usepackage{bm}
\newcommand{\bs}[1]{{\bm{#1}}}
\usepackage{graphicx,booktabs,multirow}
\usepackage{enumerate} 
\usepackage{epstopdf}
\usepackage{tikz}
\usetikzlibrary{arrows,decorations.markings}
\usepackage{caption}
\usepackage{setspace}
\usepackage{graphics}
\usepackage[labelformat=simple]{subcaption}
\usetikzlibrary{shapes}
\usepackage{amsthm}
\usepackage{cite}
\usepackage{letltxmacro}
\LetLtxMacro{\originaleqref}{\eqref}
\LetLtxMacro{\originalref}{\ref}

\definecolor{colorhkust}{HTML}{142B8C}
\definecolor{colorshanghaitech}{HTML}{A20005}
\definecolor{colortsinghua}{HTML}{743481}
\definecolor{colordark}{RGB}{184,134,11}
\definecolor{colorRed}{RGB}{128, 0, 0}
\definecolor{colorGreen}{RGB}{0, 64, 0}
\definecolor{colorBlue}{RGB}{0, 0, 128}
\usepackage[
colorlinks, 
linkcolor=colorRed, 
anchorcolor=blue, 
citecolor=colorBlue, 
urlcolor = colorGreen
]{hyperref}
\usepackage{etoolbox}

\renewcommand{\eqref}{\originaleqref}

\newcommand{\sat}{\mathsf{S}}

\newcommand{\bx}{\bs{x}}

\newtheorem{thm}{Theorem}
\newtheorem{lem}{Lemma}
\newtheorem{defn}{Definition}
\newtheorem{ass}{Assumption}

\tikzset{
	->-/.style={decoration={markings, mark=at position 0.5 with {\arrow{stealth}}}, postaction={decorate}}
}

\begin{document}

	\title{Hierarchical Learning and Computing over Space-Ground Integrated Networks}
	\author{Jingyang~Zhu,~\textit{Graduate Student Member, IEEE},~Yuanming~Shi,~\textit{Senior Member, IEEE},\\Yong~Zhou,~\textit{Senior Member, IEEE},~Chunxiao~Jiang,~\textit{Fellow, IEEE},~and~Linling~Kuang,~\textit{Member, IEEE}
	
 \thanks{Jingyang Zhu, Yuanming Shi, and Yong Zhou are with the School of Information Science and Technology, ShanghaiTech University, Shanghai 201210, China (e-mail: \{zhujy2, shiym, zhouyong\}@shanghaitech.edu.cn).
 	\textit{(Corresponding author: Yuanming Shi.)}
 	
Chunxiao Jiang and Linling Kuang are with the Beijing National Research Center for Information Science and Technology, Tsinghua University, Beijing, 100084, China (e-mail: \{jchx,  kll\}@tsinghua.edu.cn).

                
 }}
	
	\maketitle
	\IEEEpeerreviewmaketitle
    \begin{abstract}
	Space-ground integrated networks hold great promise for providing global connectivity, particularly in remote areas where large amounts of valuable data are generated by Internet of Things (IoT) devices, but lacking terrestrial communication infrastructure.
	The massive data is conventionally transferred to the cloud server
	for centralized artificial intelligence (AI) models training, raising huge communication overhead and privacy concerns.
	To address this, we propose a hierarchical learning and computing framework, which leverages the low-latency characteristic of low-earth-orbit (LEO) satellites and the global coverage of geostationary-earth-orbit (GEO) satellites, to provide global aggregation services for locally trained models on ground IoT devices.
	Due to the time-varying nature of satellite network topology and the energy constraints of LEO satellites, efficiently aggregating the received local models from ground devices on LEO satellites is highly challenging.
	By leveraging the predictability of inter-satellite connectivity, modeling the space network as a directed graph, we formulate a network energy minimization problem for model aggregation, which turns out to be a \textit{Directed Steiner Tree (DST)} problem.
	We propose a topology-aware energy-efficient routing (TAEER) algorithm to solve the \textit{DST} problem by finding a minimum spanning arborescence on a substitute directed graph. Extensive simulations under real-world space-ground integrated network settings demonstrate that the proposed TAEER algorithm significantly reduces energy consumption and outperforms benchmarks.
\end{abstract}

\begin{IEEEkeywords}
	Space-ground integrated networks, hierarchical learning, topology-aware energy-efficient routing.
\end{IEEEkeywords}

\section{Introduction}
With the rapid development of data sensing capabilities of smartphones, electric vehicles, and various sensors, a massive amount of generated data can be leveraged to train artificial intelligence (AI) models \cite{EdgeAI6G}. These AI models support various intelligent applications, such as autonomous driving \cite{katare2023ASurvey} and environment monitoring \cite{su2022retrieval}.
Due to the advanced fifth-generation (5G) communication infrastructure in urban areas, Internet of Things (IoT) devices can quickly connect with cloud servers, providing necessary support for data transmission and AI model training. 
However, remote areas such as forests, deserts, and oceans, which have significant demand for intelligent services like disaster monitoring and deep-ocean exploration \cite{li2020enabling}, lack coverage by ground base stations. This makes it challenging for data generated by IoT devices in these regions to be exploited for AI model training. To address this issue, an space-ground integrated network (SGIN), serving as a complement and extension to terrestrial communication networks, leverages space-based communication and computing capabilities to provide wireless data transmission services for intelligent applications in remote areas \cite{liu2018spaceairground,zhu2022delay}.
The SGINs comprise terrestrial IoT devices and the space computing power network (Space-CPN) \cite{shang2021computing}. Space-CPN typically includes unmanned aerial vehicles (UAVs), high-altitude platforms (HAPs), and various types of satellites. 

To accelerate the commercial deployment of space-ground communications, standards starting from 3GPP Release 17 specifications provide more support for non-terrestrial networks (NTN) \cite{3gpp_ntn_rel17}.
Meanwhile, owing to the large-scale deployment of space-ground communication hardware in IoT devices, particularly the advancements in satellite communication hardware, direct connections between IoT devices and satellites have been successfully validated \cite{le2024asurvey}.
Traditional centralized learning requires transmitting data generated by widely distributed IoT devices to ground cloud servers for centralized processing via the SGIN, leading to substantial communication overhead and severe privacy concerns.
To this end, a hierarchical learning and computing system, where remote IoT devices locally train AI models using their collected data, while the Space-CPN directly provides model aggregation and dissemination services for these IoT devices, can significantly enhance data transmission and learning efficiency \cite{Al2024federated,chen2023time}. 
This paradigm can also be known as multi-tier federated learning (FL) \cite{zhang2022scalable}.
{In the context of existing research on learning and computing in SGINs, FL in SGINs can be generally classified into two distinct categories. The first involves terrestrial devices performing local model training while satellite-based systems handle model aggregation (e.g., \cite{fang2023olive,han2024cooperative,wang2023hierarchical}), whereas the second employs space-based nodes for local model training with subsequent aggregation conducted at ground stations (e.g., \cite{lin2023fedsn,yang2024communicationefficient}). 
Moreover, recent literature also studied edge computing techniques in SGINs, particularly focusing on critical aspects including service caching mechanisms \cite{hu2025joint}, task offloading strategies \cite{zhou2025latencyenergy}, and on-orbit resource allocation schemes \cite{sun2024distributionally}.
}

Satellites in Space-CPN can generally be classified into three categories based on their orbital altitudes, i.e., geostationary-earth-orbit (GEO) satellites, medium-earth-orbit (MEO) satellites, and low-earth-orbit (LEO) satellites \cite{jiang2020reinforcement}.
To support the hierarchical learning and computing system, LEO mega-constellations, which are in active deployment nowadays (e.g., Starlink, Kuiper, Telesat, etc.) and are characterized by shorter communication delays and higher signal strengths, are highly suitable for effectively gathering local models from these remote IoT devices \cite{chaudhry2021laser}. 
GEO satellites can efficiently manage and optimize communication paths between LEO satellites and are suitable for aggregating local models across the network and broadcasting the updated global model, due to their synchronous characteristics.
Therefore, this paper considers a three-layer SGIN encompassing GEO satellites, LEO mega-constellations, and terrestrial edge devices for the implementation of learning and computing procedures,
fully leveraging the computational capabilities of on-orbit processing and the advantages of inter-satellite communications \cite{george2018onboard,ouyang2023joint}.
In terrestrial networks, the scheduling of communication schemes for distributed model training, such as data, model, and pipeline parallelism in AI center networks, is critically important \cite{deng2024cloud}. While these methods may not significantly affect the final performance of model training, they can greatly accelerate the overall training process. For instance, communication patterns between servers and nodes, such as gather, scatter, and all-reduce, are specifically designed for different types of computational tasks. Therefore, designing a communication scheme tailored to AI traffic for a three-layer SGIN, specifically a routing method for model aggregation in FL procedures, is highly meaningful.

In this three-layer SGIN, the energy consumption issue of the LEO satellites is particularly prominent.
The power system of a satellite comprised of solar panels and battery cells, which generate electricity from sunlight and store the generated power, respectively.
Due to the limitations imposed by the launch efficiency of rockets, LEO satellites must compromise in terms of size, weight, and power, with their capability of energy generation being particularly restricted \cite{wang2024dynamicLISL}. 
Meanwhile, GEO satellites, during their orbital flight around the Earth, are predominantly exposed to sunlight, whereas LEO satellites frequently transition between daylight and nighttime conditions. This cyclical exposure poses significant challenges for power management in LEO satellites, since abuse of energy results in the rapid depletion of the satellite's operational lifespan \cite{yang2016towards}.
{Specifically, the batteries of the satellites have inherent limitations in terms of recharge/discharge cycles, commonly known as the depth of discharge (DOD) cycle. These limitations restrict not only the useful life of the batteries but also that of the satellites \cite{liu2024green,jing2023energy,macambira2022energy}.
Therefore, to aggregate the local models of ground IoT devices within LEO satellite networks through inter-satellite routing, in face of pronounced energy consumption concerns, it is imperative that the routing algorithms within LEO mega-constellations be designed to optimize for high energy efficiency and low power consumption \cite{yang2016towards,alagoz2011energy}.}

{Compared to conventional terrestrial routing algorithms, routing in LEO satellite networks faces three major challenges: highly dynamic topologies, constrained on-board resources, and complex inter-satellite link (ISL) conditions. Existing studies have primarily addressed the above challenges in scenarios focused on providing relay connections for two distant terrestrial communication terminals \cite{han2022time}.}
This can be modeled as an end-to-end shortest distance path problem, aiming to reduce the number of hops in the LEO satellite networks or to minimize the transmission delay or overhead \cite{werner1997dynamic,ekici2001distributed,chen2024shortest}. 
Various methods can be applied to address these problems, including the Dijkstra algorithm, Bellman-Ford algorithm, and Floyd-Warshall algorithm \cite{gallo1988shortest}. 
These algorithms fundamentally model the satellite network as a graph and compute the shortest paths through iterative methods.
However, in view of the hierarchical learning and computing architecture, LEO satellites that collect local models from ground devices are discretely distributed across the entire LEO mega-constellation. These LEO satellites can be regarded as \textit{terminals}.
On a large time scale, due to the Earth's rotation and the mobility of LEO satellites, the identity of satellite \textit{terminals} are different across communication rounds in the hierarchical learning procedures. 
On a smaller time scale, although intra-orbit communication topology can be considered stable \cite{zeng2024satellite}, the inter-orbit communication topology changes rapidly \cite{zhai2023fedleo}.
This raises a critical challenge of efficiently aggregating these local models from \textit{terminals} through inter-satellite routing with ISLs in a highly dynamic network topology.


In this paper, we consider a hierarchical learning and computing framework over a three-layer SGIN encompassing GEO satellites, LEO mega-constellations, and terrestrial edge devices for the implementation of learning and computing procedures.
The aim of this paper is to design an energy-efficient global model aggregation routing algorithm for the hierarchical learning and computing framework based on the satellite network topology.

The major contributions are summarized as follows:
\begin{itemize}
	\item[1)] We propose a novel hierarchical learning and computing framework, tailed to SGINs, to support collaboratively training on widely distributed terrestrial IoT devices with LEO mega-constellations and GEO satellites.
	In this SGIN, LEO satellites collect local models from terrestrial IoT devices and aggregate them through ISLs. GEO satellites is responsible for managing the routing paths and broadcasting the updated global model to these IoT devices.
	\item[2)] We leverage the predictability of satellite network topology by dividing it into snapshots. We propose to model the space topology as a directed graph based on the inter-satellite connectivity in snapshots, where the weights of directed edges are set to be the energy consumption for model transmission. We reveal that minimizing the overall network transmission energy is equivalent to finding a minimum directed sub-branching covering all the terminals with the lowest energy consumption for a given topology. This is a \textit{Directed Steiner Tree (DST)} problem, which is NP-hard and more challenging in the satellite network topology.
	\item[3)] To solve the problem, we propose a topology-aware energy-efficient routing (TAEER) algorithm by first using Dijkstra algorithm to obtain the minimum cost path to the root node and then finding a minimum spanning arborescence for a substitute directed graph based on the space topology and the minimum cost paths.
	\item[4)] We extend our formulated problem and solution to a more general case where inpredictability such as transmission outage happens. By characterizing the outage probability caused by pointing error, we modify the proposed TAEER solution for feasibly applied in the robust scenarios.
\end{itemize}
In realistic SGIN settings, extensive simulations on both Walker-Star and Walker-Delta constellations are conducted. The
results demonstrate that our proposed TAEER scheme can achieve a better performance than other benchmarks, which verifies the superiority of the proposed algorithm.
%
%

The rest of this paper is organized as follows. We start by introducing the overall system model of in Section \ref{sec: system model}. 
Next, we formulate a TAEER problem for global aggregation and provide a solution to this problem in Section \ref{sec: sol1}. 
Subsequently, we provide the simulation results based on professional platforms in Section \ref{sec: simulations}. 
At last, Section \ref{sec: conclusion} summarizes the conclusions and the future research directions of this paper.



\section{System Model}\label{sec: system model}
In this section, we first present the network architecture of the SGINs, as shown in Fig. \ref{fig: systemmodel}.
We then provide detailed information on the hierarchical learning and computing framework and the communication model.

\subsection{Network Architecture}\label{sec: network}
\subsubsection{Ground Networks}
Consider a set $\mathcal{K} = \{1,\cdots,K\}$ of $K$ clusters located in different remote geographic regions, where the ground base stations and other terrestrial communication infrastructures are not available.
In addition, a set $\mathcal{J} = \{1,\cdots,J\}$ of $J$ IoT devices are assumed to be distributed across these $K$ clusters, with each IoT device being located within a single cluster. We denote $\mathcal{K}_{j}$ as the set of IoT devices in cluster $j\in\mathcal{K}$, where $\cup_{j\in\mathcal{K}}\mathcal{K}_j=\mathcal{K}$. Each IoT device $i$ in cluster $j$, i.e., $i\in\mathcal{K}_{j}$, possesses its local dataset $\mathcal{D}_{i,j}= \{\bm{\theta}_{i,j}^l,y_{i,j}^l\}_{i=1}^{|\mathcal{D}_{i,j}|}$. Let $\mathcal{D}_{j} =\cup_{i\in\mathcal{K}_j}\mathcal{D}_{i,j}$ denote the whole dataset in cluster $j$, and $\mathcal{D} = \cup_{j\in\mathcal{J}}\mathcal{D}_{j}$ as the overall dataset in the hierarchical learning system.
By implementing hierarchical learning procedures, each IoT device participates in model training without sharing their local private data, utilizing their local computation capabilities.
Specific application scenarios include large-scale geospatial applications such as environmental monitoring, ocean exploration, and smart agriculture.

\subsubsection{Space Networks}
To establish a hierarchical learning and computing framework covering all IoT devices in the remote clusters, 
we consider a LEO mega-constellation with a set of $\mathcal{N} = \{1,2,\dots,N\}$ orbit planes and three GEO satellites.
The function of space networks is to provide model aggregation and model broadcasting services for ground IoT devices through inter-satellite routing.
In orbit $n\in\mathcal{N}$, there are $S_n$ LEO satellites denoted as $ \mathcal{S}_n = \left\{\sat_{n,1}, \sat_{n,2}, \dots, \sat_{n,S_n}\right\} $ traveling around the Earth in the same direction.
The total number of LEO satellites is thus given by $S=\sum_{n=1}^{N}S_n$.
To leverage the low latency and wide coverage characteristics of LEO satellite networks to provide access services for IoT devices lacking ground communication infrastructure, LEO satellites are tasked with servicing IoT devices within their respective clusters by enabling hierarchical learning \cite{Le2024survey}.
Additionally, these networks are responsible for the fusion of information across different clusters. 
Each terrestrial cluster can be supported by the LEO satellites from different orbits, or possibly the same orbit depending on their locations.
However, due to the high mobility of LEO satellites, the LEO satellite serving a specific terrestrial cluster varies over time. 
Consequently, the concept of \textit{logical location} is introduced to define the LEO satellite that serves a ground cluster at any point in time \cite{ekici2001distributed}. 
It is assumed that the entire Earth surface is covered by \textit{logic locations}, which keep static and are taken over by the nearest satellite. 
This implies that once the coverage area of a given satellite no longer encompasses a particular terrestrial cluster, its service capabilities are handed over to the next satellite within the same orbit. 
By considering the service hand-off, the impact of satellite mobility on providing services for ground devices can be mitigated.
\begin{figure}[t]
	\centering
	\includegraphics[width=1\linewidth]{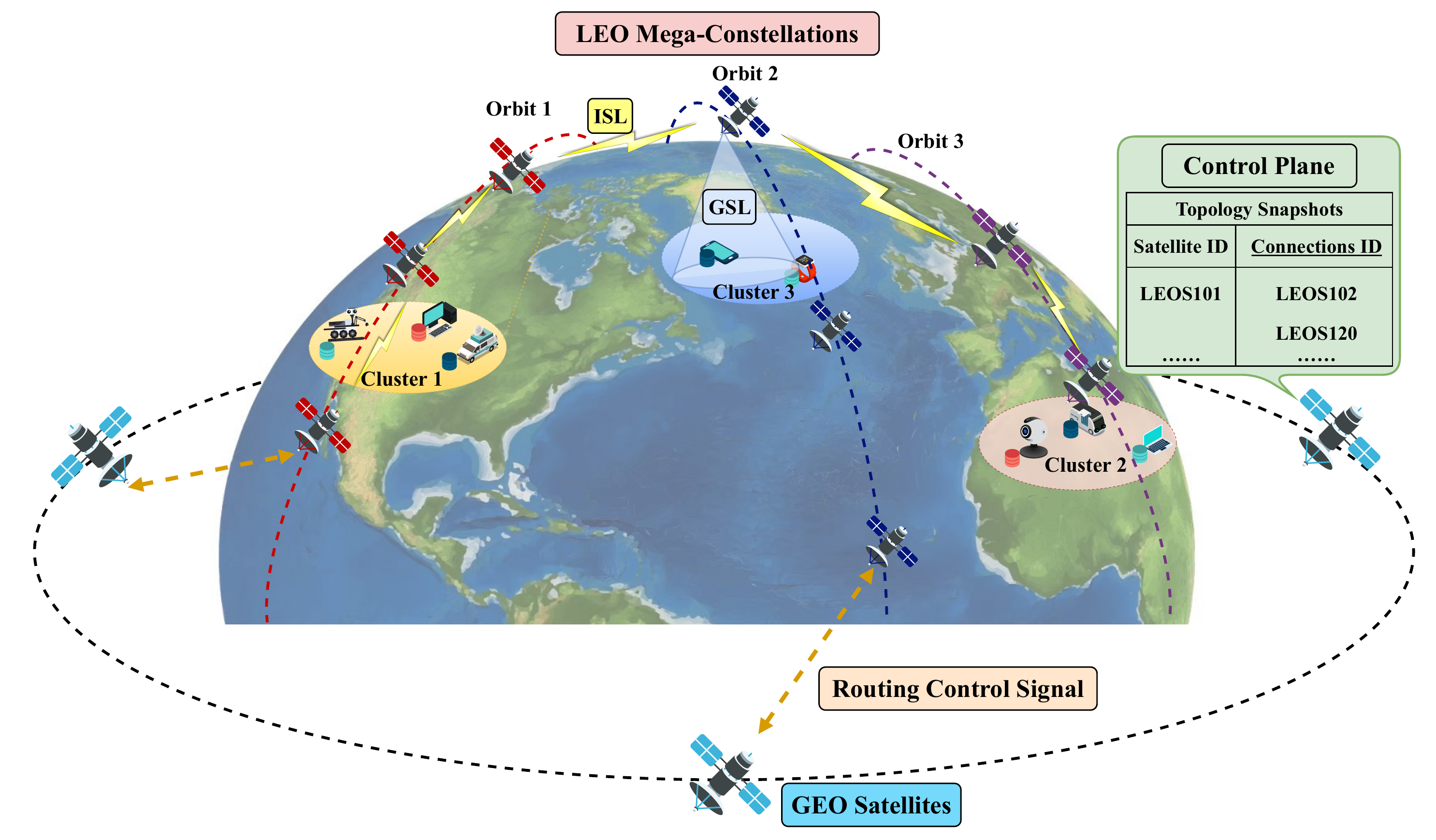}
	\caption[]{System model of a SGIN.}
	\label{fig: systemmodel}
	\vspace{-0.5cm}
\end{figure}
Three GEO satellites $\mathcal{G} = \{\mathcal{G}_1,\mathcal{G}_2,\mathcal{G}_3\}$, whose coverage encompasses the entire Earth surface, serve as the control plane, overseeing communication routing decisions among all LEO satellites and broadcasting information to all terrestrial IoT devices.
Due to the global coverage of three GEO satellites, each LEO satellite can report its current status to them, including energy consumption, ISLs conditions, and other relevant conditions.
In view of this, the routing strategy within the LEO mega-constellations can be directly managed by the control plane on the GEO satellites.

\subsection{Hierarchical Learning and Computing Framework}
In this subsection, we elaborate the hierarchical learning and computation framework over the SGINs.
First of all, the global objective is to minimize the global loss function defined as
\begin{equation}
	\min_{\bx\in\mathbb{R}^{d}}  f(\bx) := \sum_{n=1}^{J} \lambda_{n}f_{n}(\bx),
\end{equation}
where $\lambda_n\geq0$ is the aggregation weight for IoT device $n$, $f_{n}(\cdot)$ is the local loss function of IoT device $n$, and $\bx$ is the model parameter of dimension $d$. Based on the system model described in Section \ref{sec: network}, the global objective can be rewritten as
\begin{equation}
	f(\bx)= \sum_{j=1}^{K}\sum_{i=1}^{|\mathcal{K}_j|} \lambda_{i,j}f_{i,j}\left(\bx\right),
\end{equation}
where $\lambda_{i,j} = \frac{|\mathcal{D}_{i,j}|}{|\mathcal{D}|}$ is the sample size ratio and $f_{i,j}\left(\cdot\right)$ denotes the local loss function of IoT device $i$ in cluster $j$. The local loss function is defined by the learning task and can be written as
\begin{equation}
	f_{i,j}(\bx)=\frac{1}{|\mathcal{D}_{i,j}|} \sum_{l\in\mathcal{D}_{i,j}}\ell\left(\bx ; \boldsymbol{\theta}_{i,j}^{l}, y_{i,j}^{l}\right),
\end{equation}
where $ \ell(\bx;\bm{\theta}_{i,j}, y_{i,j}) $ represents the sample-wise loss function.

In communication round $t$, the hierarchical learning and computing framework, as shown in Fig. \ref{fig: hierlc}, includes two main procedures, i.e., on-device local computing and global aggregation. We assume that the learning procedures operate for $T$ communication rounds.
\subsubsection{On-Device Local Training}
At the beginning of the whole training process, i.e., $t=0$, one leading GEO satellite initializes the global model $\bm{x}^0$ and broadcasts it to all IoT devices. For $t=1,\dots,T-1$, the GEO broadcasts the updated global model $\bm{x}^{t}$ to all IoT devices.
For each IoT device $ i\in\mathcal{K}_{j} $, it first initializes the local model $\bm{x}_{i,j}^{t,0}$ as $\bm{x}_{i,j}^{t,0}:=\bm{x}^{t}$,
Then, the on-board local training process is given by performing $E$ rounds of local mini-batch SGD
\begin{align}
	\bm{x}_{i,j}^{t,e+1} = \bm{x}_{i,j}^{t,e} - \eta \tilde{\nabla} f_{i,j}(\bm{x}_{i,j}^{t,e}),~e = 0,\dots,E-1,
\end{align}
where $\tilde{\nabla}f_{i,j}(\cdot)$ denotes the local mini-batch stochastic gradient.
Consequently, IoT device $i\in\mathcal{K}_{j}$ obtains the local model difference $\Delta\bm{x}_{i,j}^t$ as
\begin{align}
	\Delta\bm{x}_{i,j}^t:=\bm{x}_{i,j}^{t,E}-\bm{x}_{i,j}^{t,0} = -\eta \sum_{e=0}^{E-1}\tilde{\nabla} f_{i,j}(\bm{x}_{i,j}^{t,e}),
\end{align}
and reports the weighted local model difference $\lambda_{i,j}\Delta\bm{x}_{i,j}^t$ to the LEO satellite.
\begin{figure}[t]
	\centering
	\includegraphics[width=1\linewidth]{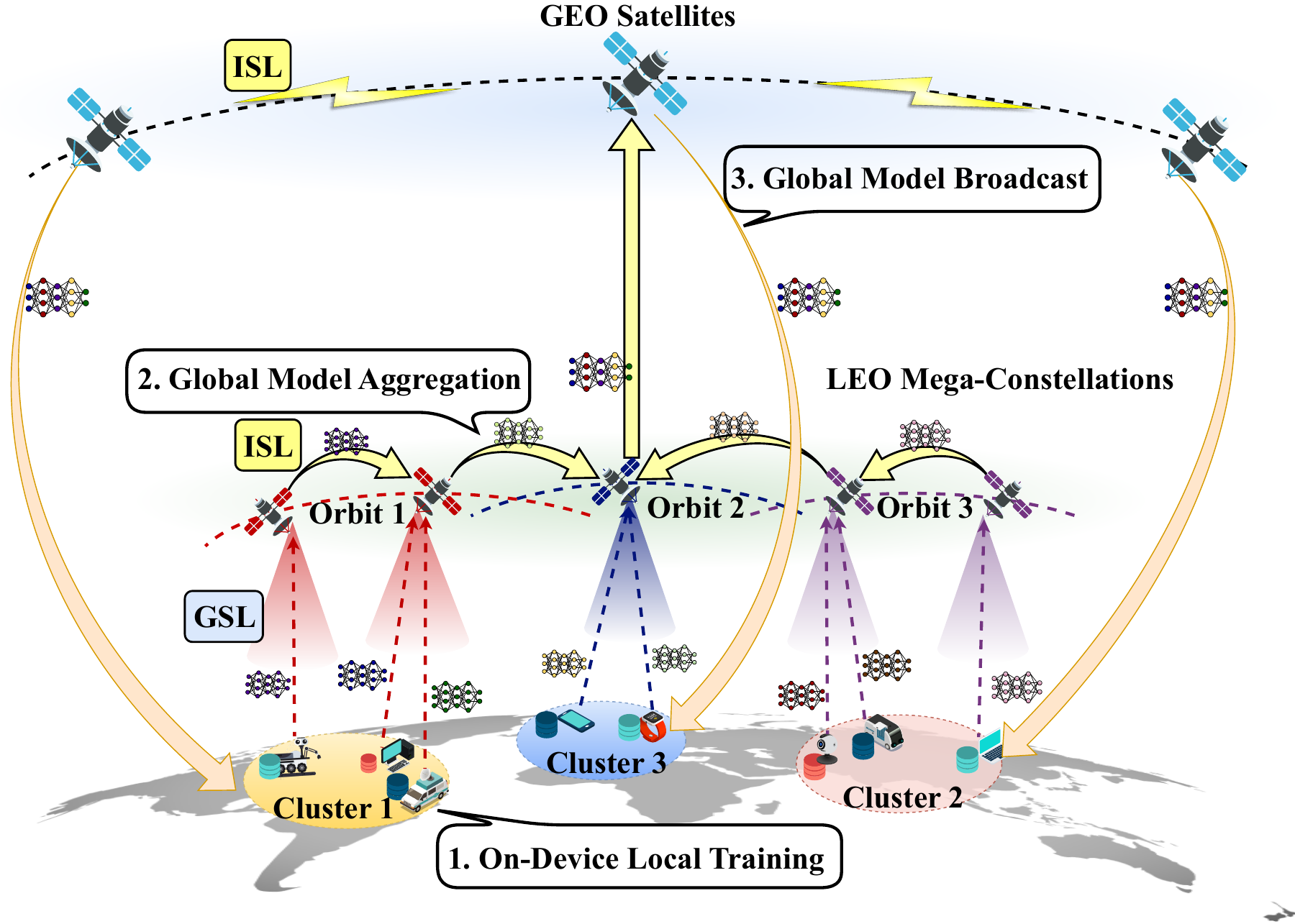}
	\caption[]{Hierarchical learning and computing framework.}
	\label{fig: hierlc}
	\vspace{-0.5cm}
\end{figure}

\subsubsection{Global Aggregation}
Upon the completion of the on-device local training phase on IoT devices, the LEO satellites that receive the local models from the ground IoT devices are distributed throughout the mega-constellations. 
This paper primarily focuses on the effective design of model aggregation methods that enable GEO satellites to acquire the updated global model.
The update of the global model in communication round $t$ can be represented as
\begin{align}
	\bm{x}^{t+1}:=\bm{x}^{t}+\sum_{j=1}^{K}\sum_{i=1}^{|\mathcal{K}_j|} \lambda_{i,j}\Delta\bm{x}_{i,j}^t.
\end{align}

Global aggregation is a typical procedure in model training that integrates communication and computation. Designing an efficient model aggregation routing method for the hierarchical learning and computing system proposed in this paper is the key to improving the efficiency of model training.

\subsection{Communication Model}
In each communication round, the required transmission involves terrestrial IoT devices uploading local models to LEO satellites and inter-satellite communications. Therefore, in this subsection, we elaborate the ground-to-satellite channel model and the inter-satellite communication model.
\subsubsection{Ground-to-Space Link (GSL)}
Upon completing the training on each IoT device, it is imperative to transmit the local model to the connected LEO satellite via GSL.
The small-scale fading between the IoT devices and the LEO satellite follows a shadowed-Rician (SR) fading model, which incorporates both the line of sight (LOS) and scatter components \cite{abdi2003new}. The probability density function (PDF) of the channel gain $|h|^2$, attributable to fading on the GSL, is delineated as follows:
\begin{equation}
	\begin{aligned}
		f_{|h|^2}(x)= & \left(\frac{2 m b_0}{2 m b_0+\Omega}\right)^m \frac{1}{2 b_0} \exp \left(-\frac{x}{2 b_0}\right) \\
		& \times_1 F_1\left(m, 1, \frac{\Omega x}{2 b_0\left(2 b_0 m+\Omega\right)}\right)
	\end{aligned}
\end{equation}
where ${}_1F_1(\cdot, \cdot, \cdot)$ denotes the hypergeometric function. The parameters $m$, $b_0$, and $\Omega$ signify the Nakagami fading coefficient, half of the scattered component's average power, and the average power of the LOS component, respectively \cite{jia2021uplink}. 		
The physical meanings of the hypergeometric function lies in the fact that when the LOS component is dominant, the contribution of the hypergeometric function becomes significant, capturing the enhancement effect of the LOS signal. Conversely, when the LOS component is weak, the impact of the hypergeometric function diminishes, reflecting the signal characteristics dominated by multipath scattering.
This expression can be closely approximated as a Gamma random variable by
\begin{equation}
	f_{|h|^2}(x) \approx \frac{1}{\beta^\alpha \Gamma(\alpha)} x^{\alpha-1} \exp \left(-\frac{x}{\beta}\right),
\end{equation}
where $\Gamma(\alpha)$ represents the Gamma function with $\alpha$ being the shape parameter and $\beta$ being the scale parameter, determined by
\begin{align}
	\Gamma(\alpha) &= \int_0^{\infty} t^\alpha \exp (-t) \mathrm{d} t, \\
	\alpha &= \frac{m(2 b_0+\Omega)^2}{4 m b_0^2+4 m b_0 \Omega+\Omega^2}, \\
	\beta &= \frac{4 m b_0^2+4 m b_0 \Omega+\Omega^2}{m(2 b_0+\Omega)}.
\end{align}
Besides the small-scale fading, the Free Space Loss due to the distance between the satellite and ground devices, can be modeled as $L=[\lambda /(4 \pi l_d)]^2$, with transmission distance $l_d$ and wavelength $\lambda$. Consequently, the channel coefficient of the communication link between IoT devices and the LEO satellite is $|H|^2 = L|h|^2$. 
We assume an uplink method based on frequency-division multiple access (FDMA) to avoid intra-cell interference and ignore the inter-cell interference signals at the receiver for simplification.
\subsubsection{Inter-Satellite Link}
In the global aggregation phase, the transmission and relay through optical ISLs between LEO and GEO satellites is pivotal. These links, operating in the vacuum of space as the propagation medium, facilitate communication between satellites stationed in space. According to \cite{liang2022link}, the received power for an optical ISL is given by
\begin{equation}\label{eq: receive power}
	P_R = P_T \eta_S G_T G_R L_{PL} L_{PS},
\end{equation}
where $P_R$ represents the received power in Watts, $P_T$ denotes the transmitted power in Watts, $\eta_S$ is the optical efficiency of the transceiver system. $G_T$ and $G_R$ represent the transmitter and receiver gains, while $L_{PL}$ is the pointing loss caused by beam misalignment. $L_{PS}$ is the free-space path loss for the optical link between satellites \cite{arnon2005performance}. The expression for the transmitter gain, $G_T$, is given by
\begin{equation}
	G_T = \frac{16}{\Theta_T^2},
\end{equation}
where $\Theta_T$ is the full transmitting divergence angle in radians \cite{polishuk2004optimization}. The receiver gain, $G_R$, is given by
\begin{equation}
	G_R = \left(\frac{D_R \pi}{\lambda}\right)^2,
\end{equation}
with $D_R$ as the diameter of the receiver’s telescope in millimeters \cite{polishuk2004optimization}. The pointing loss, denoted as $L_{PL}$, is
\begin{equation}
	L_{PL} = \exp\left(-G_0 \theta_0^2\right),
\end{equation}
where $\theta_0$ is the pointing error in radians and $G_0 =4\ln2/\theta_{\text{3dB}}^2$ is the coefficient related to the 3-dB beamwidths \cite{nie2021channel}.
Lastly, the free-space path loss, denoted as $L_{PS}$, is calculated as
\begin{equation}
	L_{PS} = \left(\frac{\lambda}{4\pi d_{SS}}\right)^2,
\end{equation}
where $\lambda$ signifies the operating wavelength in nanometers, and $d_{SS}$ denotes the inter-satellite distance in kilometers \cite{arnon2005performance}.

This paper focuses more on the scheduling of inter-satellite transmission links based on the space topology for global model aggregation in the following section.

\section{Topology-Aware Energy-Efficient Routing for Global Model Aggregation}\label{sec: sol1}
LEO satellites face significant energy management challenges due to their limited power generation capabilities and frequent mobility between daylight and eclipse regions \cite{yang2016towards}.
Therefore, routing algorithms in LEO mega-constellations must prioritize high energy efficiency and low power consumption \cite{alagoz2011energy}.
In this section, we study the topology-aware energy-efficient routing scheme in the global aggregation phase based on the space topology. 
\subsection{Characterization of Space Topology}
Space topology describes the connectivity among the LEO satellites and that between LEO and GEO satellites. 
The key idea is that the LEO satellite constellation's orbital period is divided into a series of time slots, during which the topology can be considered fixed \cite{werner1997dynamic,liu2013routing}. A snapshot is taken from a fixed topology to be observed by the GEO for manipulating the routing scheme. 
Given the predictability of space topology, all snapshots can be calculated and stored in advance \cite{han2022time}.

We assume that the complete topology cycle is periodic with period $P$, i.e., repeating the $M$ topology snapshots periodically, where $M$ is the number of topology snapshots within a constellation period. In this case, we have the following definition.

\begin{defn}[Topology Snapshot]
	A topology snapshot $p = m\Delta p,~m\in\{0,1,\cdots,M-1\}$ with $\Delta p= P/M$ can be modeled as a directed graph $\bm{G}(m) = (\mathcal{V},\mathcal{E}(m))$, where $\mathcal{V} = \cup_{n\in\mathcal{N}}\mathcal{S}_n\cup\mathcal{G}$ is the set of all satellites as nodes in the graph, and $\mathcal{E}(m)$ is the set of directed edges $(i,j)_{m} \neq (j,i)_{m}$ between node $i$ and $j$ in $\mathcal{V}$ for snapshot $m$.
\end{defn}
\begin{figure}
	\centering
	\includegraphics[width=1\linewidth]{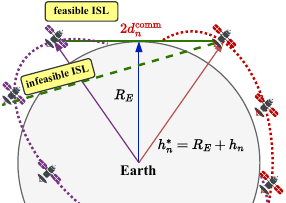}
	\caption{Distance constraint for inter-orbit laser ISLs.}
	\label{fig: ISL}
	\vspace{-0.5cm}
\end{figure}

In addition, the predictability of satellite network topology is primarily based on the following conditions.
First, it is assumed that the operational velocity and relative distance between consecutive satellites within the same orbit remain constant. 
Consequently, the establishment of intra-orbit laser ISLs between two adjacent LEO satellites in one orbit is always feasible.
Second, to establish inter-orbit laser ISLs between satellites on different orbits, visibility considerations are imperative. We assume that at a given moment, the right ascension of ascending node of the satellite $k$ on orbit $n$, the longitude angle, and the orbital inclination are denoted as $\alpha_{n,k}$, $\gamma_{n}$, and $\kappa_{n}$, respectively. The position of the satellite in geocentric coordinates $(x_{n,k},y_{n,k},z_{n,k})$ can be represented as:
\begin{align*}
		\left\{\begin{aligned}
			x_{n,k}&=h_n^*\cos \alpha_{n,k} \cos\gamma_{n}-h_n^*\sin\alpha_{n,k} \cos \kappa_{n} \sin \gamma_{n}, \\
			y_{n,k}&=h_n^* \cos\alpha_{n,k}\sin \gamma_{n} + h_n^*\sin\alpha_{n,k} \cos \kappa_{n} \sin \gamma_{n},\\
			z_{n,k}&=h_n^*\sin\alpha_{n,k} \cos \kappa_{n},
		\end{aligned}\right.
\end{align*}
where $h_n^* = R_E + h_n$  denotes the distance from satellite $\sat_{n,k}$ to the Earth center, $R_E$ is the Earth radius, and $h_n$ is the altitude of orbit $n$.
Then the distance between two satellites in different orbits with coordinates $(x_i,y_i,z_i)$ and $(x_j,y_j,z_j)$ can be obtained by
\begin{align*}
	d_{i,j} = \sqrt{(x_i-x_j)^2+(x_i-x_j)^2+(x_i-x_j)^2}.
\end{align*}
Moreover, the communication radius of a certain satellite in orbit $n$ can be given by
\begin{align}
	d_n^{\text{comm}} = 2\sqrt{(h_n^*)^2-R_E^2}.
\end{align}
To conclude, the inter-orbit laser ISLs can be established if and only if the following condition is met:
\begin{align}
	d_{i,j}\leq\min\{d_{n_i}^{\text{comm}},d_{n_j}^{\text{comm}}\}.
\end{align}
This indicates that the distance between two satellites in different orbits should not exceed the minimum communication range of the satellites if they want to establish connection with each other, as illustrated in Fig. \ref{fig: ISL}.
The network topology constraints are defined as follows. While distance constraints may be satisfied, two satellites within the same orbit cannot establish a link if there is an intermediate satellite between them, i.e., $(\sat_{n,i},\sat_{n,j})\in\mathcal{E}(m),|i-j|= 1,|i-j|=S_n-1$. For satellites in different orbits, when multiple satellites in another orbit satisfy the distance constraint, only the link with the nearest satellite is considered.
In a nutshell, the GEO satellites maintain awareness of the edges that can be established by ISLs in each snapshot. This forms a snapshot table, which can be accessed and reviewed by the GEO satellites at any time.

Moreover, we have the following definition for a time slot.
\begin{defn}[Time Slot]
	A time slot is a time interval $[m\Delta p, (m+1)\Delta p),~m\in\{0,1,\cdots,M-1\}$, where a topology snapshot remains fixed. The topology snapshots in different time slots vary. 
\end{defn}
The duration of a time slot in LEO networks can be assumed to be ranging from approximately 30 to 400 seconds \cite{wang2007topological}.

\subsection{Problem Formulation}
In this subsection, we formulate a topology-aware energy-efficient routing problem in the space networks for global aggregation.
Given the topology snapshot of a specific time slot, 
the GEO satellites can query the communication links that can be established within the space networks, i.e., the edges of graph $\mathcal{E}(m)$. Based on the current communication system configurations of LEO satellites, the weight of each edge in each snapshot, $w_{i,j}(m),~i,j \in \mathcal{V}$, can be set as the energy consumption of the ISL link with a specific direction.
Based on the description of ISL in Section \ref{sec: system model}, the achievable rate of the ISL between satellites $i$ and $j$ ($i$ is the transmitter and $j$ is the receiver) can be given as
\begin{align}
	\gamma_{i,j}(m)  = B_{i,j}(m) \text{log}_2 \left(1 + \frac{P_R^{j}(m)}{\sigma^2} \right),
\end{align}
where $B_{i,j}(m)$ is the bandwidth between nodes $i$ and $j$ in snapshot $m$, and $P_R^{j}(m)$ is the received power of node $j$ in snapshot $m$.
According to \cite{nie2021channel}, the bandwidth of the space optical systems is usually $0.02f_c$, where $f_c$ is the carrier frequency (e.g., 193 THz).
The noise power can be derived by
\begin{align}
	\sigma^2 = k_bB(T_s + T_0 + T_{\text{CMB}}),
\end{align}
where $k_b$ is the Boltzmann constant in J/K, $T_s$ is the solar brightness temperature in K, $T_0$ is the system noise temperature in K, and $T_{\text{CMB}}$ is the Cosmic Microwave Background (CMB) temperature in K \cite{nie2021channel}.
The energy consumption is thus given by
\begin{align}
	w_{i,j}(m)  =  \frac{sP_T^{i}(m)}{\gamma_{i,j}(m)},
\end{align}
where $s$ denotes the total number of bits to be transmitted and $P_T^{i}(m)$ is the transmit power of node $i$ in snapshot $m$.
\begin{figure}
	\centering
	\includegraphics[width=1\linewidth]{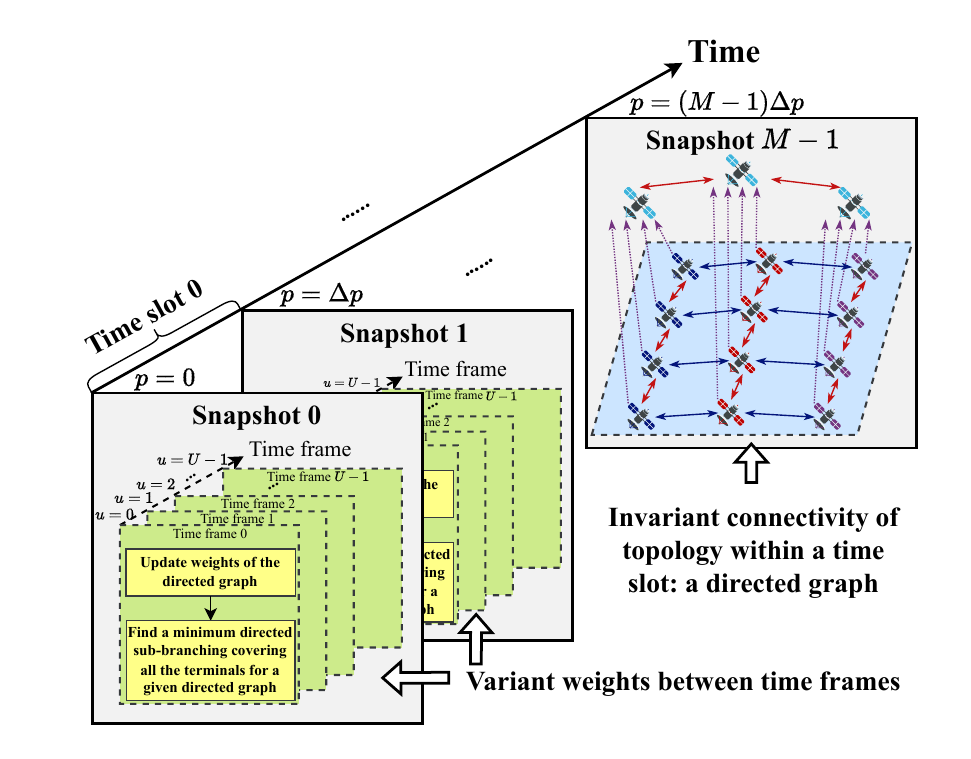}
	\caption{The space topology characterization includes both LEO and GEO satellites.}
	\label{fig: timeline}
	\vspace{-0.5cm}
\end{figure}

In the global aggregation phase, a portion of LEO satellites (i.e., \textit{terminals}), denoted as the subset of nodes $\mathcal{V}'\subset\mathcal{V}$, obtain the local models from the terrestrial clusters. 
To effectively aggregate these distributed local models over the LEO mega-constellations, we assume that the aggregation process can be completed within a certain time slot. To proceed, we define the following terms.
\begin{defn}[Directed Sub-Branching] 
	A directed sub-branching of a directed graph $\bm{G}(m) = (\mathcal{V},\mathcal{E}(m))$ is a subgraph $\mathcal{B}(m)=(\mathcal{V}',\mathcal{E}'(m))$ of the directed graph $\bm{G}(m)$ such that
	\begin{itemize}
		\item The subgraph $\mathcal{B}$ contains a set of terminals $\mathcal{V}'\subset\mathcal{V}$.
		\item The root node is not the source node of any edge in $\mathcal{E}'(m)$.
		\item For any vertex $i\in\mathcal{V}'$ other than the root node, there exists exactly one vertex $j\in\mathcal{V}'$ such that vertex $i$ is the terminal of edge $(j,i)$.
		\item The value of the directed sub-branching is given by the summation of the weights.
	\end{itemize}
\end{defn}

Based on the definition of the directed sub-branching, LEO satellites that receive the local model from terrestrial clusters can be regarded as the terminals of the directed sub-branching.
A root node denoted by $r$ can then be chosen from the terminals $\mathcal{V}'$ for the directed sub-branching.
In this case, to minimize the overall transmission energy among the LEO satellites and the GEO satellites in snapshot $m$ while aggregating all the local models, the optimization problem is given by
\begin{align}\label{eq: problem}
	\mathscr{P}_1:~\operatorname{minimize} ~& W(\mathcal{B}(m))=\sum_{i,j\in\mathcal{V}'} w_{i,j}(m),
\end{align}
where $w_{i,j}(m)$ can be regarded as the weight of edge $(i,j)$ in $\mathcal{E}'(m)$.
The objective is to find a minimum directed sub-branching covering all the terminals with the lowest energy consumption. The constraint is that the edges of the sub-branching must be part of the directed graph $\bm{G}(m)$ constructed by the current network topology. Essentially, it involves identifying an inter-satellite routing path at minimum energy cost that takes the form of a directed sub-branching.

In each time slot, the connectivity of the topology snapshot remains unchanged. However, due to satellite mobility, the relative position between two satellites might change and the corresponding channel would be different, resulting in variations in weight $w_{i,j}(m)$.
We have the following definition to further transform the problem formulation.

\begin{defn}[Time Frame]
	A time frame $\Delta \tau = \Delta p/U$ represents a temporal interval within snapshot $m$, specifically $[m\Delta p + u\Delta \tau, m\Delta p + (u+1)\Delta \tau),~u\in\{0,1,\cdots,U-1\}$, where $U$ denotes the number of time frames in a time slot. Within each time frame, the weights of graph $\bm{G}(m)$ remain constant; however, the weights may vary between different time frames.
\end{defn}
\begin{figure*}
	\centering
	\includegraphics[width=1\linewidth]{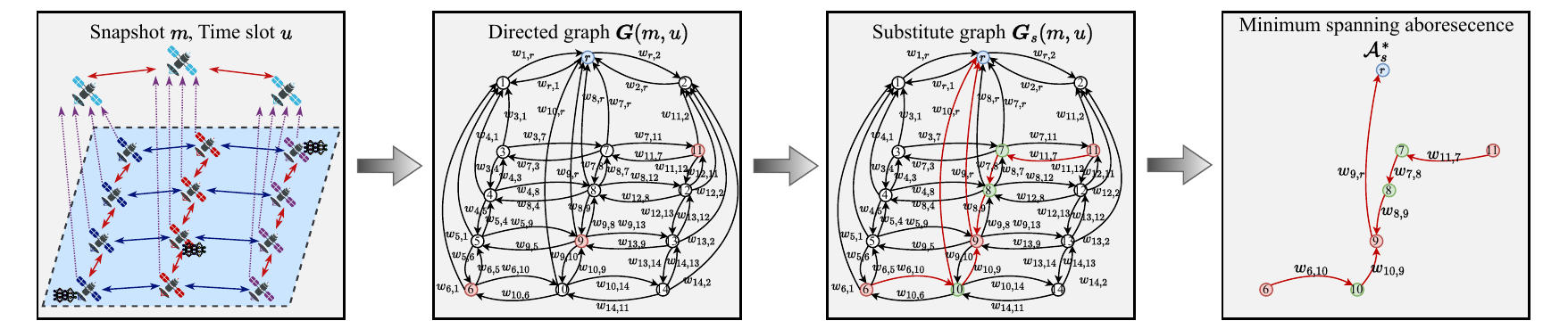}
	\caption{An example of the proposed TAEER algorithm.}
	\label{fig: solution}
	\vspace{-0.3cm}
\end{figure*}

Based on the definition of time frame, graph $\bm{G}(m)$ can be divided into $U$ subgraphs, i.e., $\bm{G}(m,u) = (\mathcal{V},\mathcal{E}(m,u)), u=0,1,\dots,U-1$. The connectivity of these subgraphs remains unchanged. However, the weights of the edges vary across subgraphs.
$\mathscr{P}_1$ can be reformulated as the following TAEER problem for each subgraph:
\begin{align}\label{eq: problem2}
	\mathscr{P}_2:~\operatorname{minimize} ~& W_u(\mathcal{B}(m,u))=\sum_{i,j\in\mathcal{V}'} w_{i,j}(m,u),
\end{align}
to minimize the energy consumption in each time frame. 
Here $\mathcal{B}(m,u)$ denotes the directed sub-branching of each time frame.
The constraint is that the edges of the directed sub-branching must be part of the subgraph $\bm{G}(m,u)$ constructed in each time frame.
This means that the solution to $\mathscr{P}_1$ can be obtained by calculating the solution to each time frame by solving $\mathscr{P}_2$. It is worth noting that the local information on each satellite is assumed to be partitioned into $U$ parts, so that each part can be transmitted within the duration of a time frame.
In this case, the energy consumption for ISLs in each time frame is given by
\begin{align}\label{eq: weight frame}
	w_{i,j}(m,u)  =  \frac{sP_T^{i}(m)}{U\gamma_{i,j}(m,u)}.
\end{align}

Based on these assumptions, $\mathscr{P}_2$ is a static problem within each time frame. The solution set to $\mathscr{P}_2$ is not finite because, although the operation of LEO constellations exhibits periodicity, the Earth's rotation causes the identity of \textit{terminals} to change across communication rounds in the hierarchical learning procedures. The overall characterization of space topology is demonstrated in Fig. \ref{fig: timeline}.

\begin{algorithm}[t]
	\caption{Proposed TAEER Algorithm}
	\label{alg: chu-liu}
	\begin{algorithmic}[1]
		\State \textbf{Initialize}: Root node $r$, graph $\bm{G}=(\mathcal{V},\mathcal{E})$, weights $w_{i,j},i,j\in\mathcal{V}$, and terminals $\mathcal{V}'$.
		\State \textbf{Input}: Graph $\bm{G}$, weights $w_{i,j}$
		\State \Comment{Shortest Path}
		\For{$i \in \mathcal{V}\setminus r$}
		\State	$\operatorname{path}\gets \operatorname{Dijkstra}(i,r)$
		\EndFor
		\State \Comment{Construct a substitute graph}
		\State $\bm{G}_s=(\mathcal{V}_s,\mathcal{E}_s)$ with weights $w^s_{i,j}$ via \eqref{eq: graph construction}.
		\State \Comment{Remove all edges leading back to root node $r$}
		\For{$(i, j) \in \mathcal{E}_s$}
		\If{$j == r$}
		\State $\mathcal{E}_s.\operatorname{remove}((i, j))$
		\EndIf
		\EndFor
		
		\State \Comment{For each vertex, find the minimum incoming edge}
		\For{$v \in \mathcal{V}'$}
		\State $\operatorname{edges} \gets \{e \in 
		\mathcal{E}_s \mid e[1] == v\}$
		\State $\pi[v] \gets \arg\min_{(ii, jj) \in \operatorname{edges}} w_{ii,jj}$
		\EndFor
		\While{\textbf{Detect} cycles $C$}
		\If{no cycles}
		\State \textbf{return} $\mathcal{A}_s'\gets$ $(\mathcal{V}_s',\mathcal{E}_s',\{w_{i,j}'\},r)$
		\Else{ \textbf{Contract} cycles in $\bm{G}_s$ forming new graph $\bm{G}_s'$}

		\For{$(i, j) \in \mathcal{E}_s$}
		\If{$i\notin C \&\& j\in C$}
		\State $e \gets (i, v_C)$		
		\State $\mathcal{E}_s'.\operatorname{add}(e)$
		\State $w_{i,v_C}' \gets w^s_{i,j} - w^s_{\pi(j),j}$
		
		\ElsIf{$i\in C \&\& j\notin C$}
		\State $e \gets (v_C, j)$		
		\State $\mathcal{E}_s'.\operatorname{add}(e)$
		\State $w_{v_C,j}' \gets w^s_{i,j} $
		
		\ElsIf{$i\notin C \&\& j\notin C$}
		\State $\mathcal{E}_s'.\operatorname{add}((i,j))$
		\State $w_{i,j}' \gets w^s_{i,j} $
		\EndIf
		\EndFor
		\EndIf
		\EndWhile
		\State\Comment{Breaking the cycle}
		\For{$(i, j) \in \mathcal{A}_s'$}
		\If{$j == v_C$}
		\State $\mathcal{E}_s'.\operatorname{remove}(\pi(j), j)$
		\State \textbf{break}
		\EndIf
		\EndFor
		\State \Return $\mathcal{A}^{\ast}$ for the approximation of $\mathcal{B}^{\ast}$.
	\end{algorithmic}
\end{algorithm}

\subsection{Proposed Topology-Aware Energy-Efficient Routing}
The topology-aware energy-efficient routing problem \eqref{eq: problem2} for global model aggregation can be modeled as a \textit{DST} problem $\mathcal{B}^{\ast}(m,u)$ in a directed graph $\bm{G}(m,u)$ \cite{zosin2002on,feldman2006the}.
The \textit{DST} problem is known to be NP-hard and the research on this problem is well-established, with many existing works providing approximations to the original problem.
For instance, a trivial algorithm for \textit{DST} problem is to obtain shortest paths from all terminals to the root and combine them \cite{CHARIKAR1999approximation}.
However, there exist several challenges for solving the \textit{DST} problem in the space networks. Specifically, the identity of satellite \textit{terminals} are different in different communication rounds and the connections between terminals and the root node cannot be guaranteed to have LOS connections, which means that relays must be used in this problem. To this end, we will introduce describe the design of a heuristic algorithm to reduce and solve the \textit{DST} problem in detail.

The main steps of the proposed solution to the energy-efficient routing problem based on \textit{DST} can be summarized as follows:
\begin{enumerate}
	\item \textit{Graph Initialization}: In time frame $u$ of time slot $m$, the graph $\bm{G}(m,u) = (\mathcal{V},\mathcal{E}(m,u))$ are initialized with weighted edges. Ensure that all terminals $\mathcal{V}'$ and the root node $r$ are included in the graph.
	\item \textit{Shortest Path}: Find the shortest distance path from each terminal node $i\in\mathcal{V}'\setminus r$ to the root node $j = r$ using the Dijkstra algorithm \cite{gallo1988shortest}, and record each predecessor node such that each terminal can trace back to the root node. This is because, compared to other similar algorithms, Dijkstra's algorithm exhibits advantages in stability and computational efficiency for solving single-source shortest-path problems on graphs with non-negative edge weights.
	Construct a set $\mathrm{Path} = \{i\rightarrow r|i\in\mathcal{V}\setminus r\}$ consisting of node and edge information for all shortest paths.
	\item \textit{Construct a Substitute Graph}: Add all the collected path edges to a new substitute subgraph $\bm{G}_s(m,u) = (\mathcal{V}_s,\mathcal{E}_s(m,u))$, which includes all terminal nodes and their paths to the root node, which consist of predecessor nodes on the path. In addition, we denote all the nodes in all paths as $\mathcal{V}_p$ and we can obtain 
	\begin{align}\label{eq: graph construction}
		\begin{aligned}
			\mathcal{V}_s &= \mathrm{Path}.\mathrm{Node}\cup r, \\
			\mathcal{E}_s(m,u) &= \mathcal{E}_{i,j}(m,u), (i,j)\in \mathrm{Path}.\mathrm{Edge}.
		\end{aligned}
	\end{align}
	\item \textit{Minimum Spanning Arborescence (MSA)}: Find the minimum spanning arborescence $\mathcal{A}^*(m,u)$ for the substitute subgraph $\bm{G}_s(m,u)$ to approximate the Directed Sub-Branching $\mathcal{B}^*(m,u)$. 
	If there are redundant branches at the ends of the spanning arborescence, pruning can be performed until the spanning arborescence is minimized.
\end{enumerate}

The key of the above steps lies in how to find a MSA for a substitute directed graph $\bm{G}_s(m,u)$.
In fact, Chu-Liu-Edmonds algorithm \cite{chu1965shortest,edmonds1967optimum} can be applied to efficiently solve the MSA problem.
We first have the definition on the arborescence.

\begin{defn}[Arborescence \cite{chu1965shortest}]
	An arborescence of a directed graph $\bm{G}_s(m,u) = (\mathcal{V}_s,\mathcal{E}_s(m,u))$ is a subgraph $\mathcal{A}_s(m,u)=(\mathcal{V},\mathcal{E}_s'(m,u))$ of the directed graph such that
	\begin{itemize}
		\item The subgraph $\mathcal{A}_s$ contains all the vertices in $\bm{G}_s$.
		\item The root node is not the source node of any edge in $\mathcal{E}_s(m,u)$.
		\item For any other vertex $i\in\mathcal{V}_s$, there exists only one vertex $j\in\mathcal{V}_s$ such that vertex $i$ is the terminal of edge $(j,i)$.
	\end{itemize}
\end{defn}
\begin{figure*}
	\centering
	\includegraphics[width=1\linewidth]{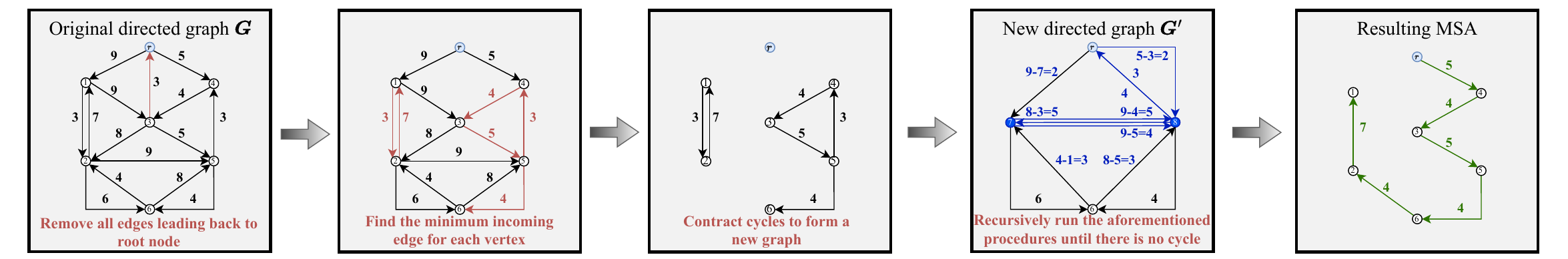}
	\caption{An example of the Chu-Liu-Edmonds algorithm.}
	\label{fig: chuliu}
	\vspace{-0.3cm}
\end{figure*}

Next, we will introduce the main steps of the Chu-Liu-Edmonds algorithm for solving the MSA problem.
\begin{enumerate}
	\item Initialization: In time frame $u$ of time slot $m$, the graph $\bm{G}_s(m,u) = (\mathcal{V}_s,\mathcal{E}_s(m,u))$ and the weights $w^s_{i,j}(m,u),i,j\in\mathcal{V}'$ are initialized on the root node $r$.
	\item Remove all edges leading back to root node $r$.
	\item For each node $v\in \mathcal{V}_s\setminus r$, find an incoming edge with the lowest weight and record the source as $\pi(v)$. If this procedure forms an arborescence, we have all the minimum edges and the algorithm can be jumped to step 5). Otherwise, there exists
	a cycle $C$.
	\item \textit{Contract}: A cycle $C$ is then contracted into a new virtual node $v_C$ and the related edges should be adjusted to form a new directed graph $\bm{G}_s'= ( \mathcal{V}_s',\mathcal{E}_s')$. The contract procedure follows the  rules below.
	\begin{itemize}
		\item The nodes of $\mathcal{V}_s'$ are the nodes of $\mathcal{V}_s$ not in $C$ plus a new virtual node $v_C$. 
		\item If an edge $(i,j)$ in $\mathcal{E}_s$ has a source out of $C$ ($i\notin C$) and a destination within $C$ ($j\in C$), i.e., a incoming edge into the cycle, then a new edge $e = (i,v_C)$ should be included in $\mathcal{E}_s'$ with weight $$w_{i,v_C}' = w^s_{i,j} - w^s_{\pi(j),j}.$$
		\item If an edge $(i,j)$ in $\mathcal{E}_s$ has a source within $C$ ($i\in C$) and a destination out of $C$ ($j\notin C$), i.e., an edge going away from the cycle, then a new edge $e = (v_C,j)$ should be included in $\mathcal{E}_s'$ with weight $$w_{v_C,j}' = w^s_{i,j}.$$
		\item If an edge $(i,j)$ in $\mathcal{E}_s$ has both source and destination out of $C$ ($i,j\notin C$), i.e., an edge unrelated to the cycle, then a new edge $e = (i,j)$ should be included in $\mathcal{E}_s'$ with weight $$w_{i,j}' = w^s_{i,j}.$$
	\end{itemize}
	By recursively calling the above operations until there is no cycle, we can obtain a MSA $\mathcal{A}_s'$.
	\item \textit{Expansion}: All non-root nodes including virtual nodes have one minimum incoming edge $e$. Remove the edge $(\pi(j),j)$ for $j\in C$, breaking the cycle. Mark each remaining edge in $C$ and each corresponding edge of $\mathcal{A}_s'$ in $\mathcal{E}_s'$, which form a MSA $\mathcal{A}^{\ast}(m,u)$.
\end{enumerate}
The detailed procedures of the proposed TAEER algorithm are presented in Algorithm \ref{alg: chu-liu}, an simple example of the solution is illustrated in Fig. \ref{fig: solution}, and an example flowchart of Chu-Liu-Edmonds algorithm for solving the MSA problem is illustrated in Fig. \ref{fig: chuliu}.
Based on the derived MSA $\mathcal{A}^{\ast}(m,u)$, a tree structure in such an LEO satellite network enables transmission-and-aggregation from each leaf node to the root node, ultimately completing the global model aggregation task.

To derive the time complexity of the proposed topology routing algorithm in each time frame, we have the following analysis:
\begin{itemize}
	\item The time complexity of Dijkstra algorithm is given by $\mathcal{O}(|\mathcal{V}|+|\mathcal{E}|\log|\mathcal{E}|)$. In the worse case, the algorithm needs to run $|\mathcal{V}'|$ times, which is the number of the terminals, and the total time complexity is given by $\mathcal{O}\left(|\mathcal{V}'|(|\mathcal{V}|+|\mathcal{E}|\log|\mathcal{E}|)\right)$. This means that the complexity is determined by the size of the space topology as well as the number of the terminals, i.e., the number of LEO satellites receiving local models from terrestrial IoT devices.
	\item The time complexity of obtaining the MSA is given by $\mathcal{O}(|\mathcal{E}_s|\log|\mathcal{V}_s|)$, which is decided by the size of the substitute graph.
	\item In summary, the overall time complexity of Algorithm \ref{alg: chu-liu} is given by $\mathcal{O}\left(|\mathcal{V}'|(|\mathcal{V}|+|\mathcal{E}|\log|\mathcal{E}|)+|\mathcal{E}_s|\log|\mathcal{V}_s|\right)$, which is dominated by the complexity of Dijkstra algorithm since $|\mathcal{V}_s|\leq|\mathcal{V}|$ and  $|\mathcal{V}'|\leq|\mathcal{V}|$.
\end{itemize}

\subsection{Extension to Unpredictable Topology}
In this subsection, we further consider some practical issues such as transmission failures caused by outages, and extend the problem formulation and solution to a more general case by balancing the energy efficiency and the robustness of the system.
On one hand, in the previous subsection, we assume that all the topology changes can be predicted in advance. 
However, there may exist some unpredictable anomalies such as satellite failure, transmission failure, and link congestion in the space topology, which cannot be avoided in practice \cite{liu2013routing,pan2022latency}.
On the other hand, the algorithm presented in Algorithm \ref{alg: chu-liu} is highly vulnerable to link failures in the MSA. 
If an edge in the spanning arborescence breaks, i.e., if a transmission failure occurs, the global model aggregation will lose part of the local model information.
To address these issues and better adapt to practical space environments, we revisit the formulation of $\mathscr{P}_1$ by incorporating realistic constraints as follows.


When solving $\mathscr{P}_1$ and $\mathscr{P}_2$, we assume that the transmission over each hop of the laser ISL is perfect. 
However, in practical systems, a pointing, acquisition, and tracking (PAT) system is required to maintain precise link alignment to ensure the LOS characteristic.
Satellite positioning inaccuracies, environmental disturbances, and mechanical vibrations are the main contributing factors to the misalignment of the telescope beam \cite{arnon2005performance,zhu2025passive}.
Misalignment of the telescope beam can lead to pointing errors \cite{arnon2005performance,polishuk2004optimization}, which in turn may cause transmission outages. This introduces uncertainties that we aim to address.
We consider transmission failure caused by outage in the laser ISL. To characterize the outage probability of the laser ISL, it can be modeled as the probability that the signal-to-noise ratio (SNR) falls below a certain threshold, which can be expressed as
\begin{align}
	P_{\text{out}}({\sf SNR}) = \mathbb{P}\{{\sf SNR}<{\sf SNR}_{\text{th}}\},
\end{align}
where ${\sf SNR}_{\text{th}}$ is the threshold of the SNR.
We denote the SNR of the ISL between nodes $i$ and $j$ as ${\sf SNR}_{ij}$ and it can be derived based on the received power \eqref{eq: receive power} by
\begin{align}
	{\sf SNR}_{i,j} = \frac{P_T^i \eta_S G_T^i G_R^j L_{PL}^{i,j} L_{PS}^{i,j}}{\sigma^2},
\end{align}
where $L_{PS}^{i,j} = (\frac{\lambda}{4\pi d_{i,j}})^2$.
We can further express this outage probability as
\begin{align}\label{eq: outage1}\nonumber
	P_{\text{out}}^{i,j}({\sf SNR}) &= \mathbb{P}\left\{{\sf SNR}_{i,j}<{\sf SNR}_{\text{th}}\right\}\\
	& = \mathbb{P}\left\{L_{PL}^{i,j}<\frac{\sigma^2{\sf SNR}_{\text{th}}}{P_T^{i}\eta_S G_T^i G_R^jL_{PS}^{i,j}}\right\}.
\end{align}
Moreover, the pointing error $\theta_0$ is assumed to follow the Rayleigh distribution \cite{chen1989impact}, and its probability density function (PDF) is given by:
\begin{align}\label{eq: Rayleigh}
	f_{\theta_0}(\theta_0) = \frac{\theta_0}{\sigma_p^2}\exp\left(-\frac{\theta_0^2}{2\sigma_p^2}\right),~\theta_0\geq0,
\end{align}
where $\sigma_p$ is a scale parameter of the distribution.
In addition, the cumulative distribution function (CDF) of the Rayleigh distribution is
\begin{align}
	F_{\theta_0}(\theta_0) = \int_{0}^{\theta_0} f_{\theta_0}(t)dt=1-\exp\left(-\frac{\theta_0^2}{2\sigma_p^2}\right).
\end{align}
\begin{lem}\label{lem: 1}
	The PDF of the pointing loss $L_{PL}$ is given by
	\begin{align}
		f_{L_{PL}}(\vartheta ) = \frac{\sigma_p\sqrt{G_0}\vartheta^{\frac{1}{2G_0\sigma_p^2}-1}}{\sqrt{-\pi\ln\vartheta}},~\vartheta \in(0,1).
	\end{align}
\end{lem}
\begin{proof}
We can obtain the following probability:
\begin{align}\nonumber
	\mathbb{P}\{L_{PL}\leq\vartheta\}=&\mathbb{P}\{\exp(-G_0\theta_0^2)\leq\vartheta\}\\\nonumber
	=&\mathbb{P}\left\{\theta_0>\sqrt{\frac{-\ln\vartheta}{G_0}}\right\}\\\nonumber
	=&1-F_{\theta_0}\left(\sqrt{\frac{-\ln\vartheta}{G_0}}\right)\\
	\overset{(a)}{=}&1-\mathrm{erf}\left(\sqrt{\frac{-\ln\vartheta}{2G_0\sigma_p^2}}\right),
\end{align}
where $(a)$ holds by the fact that $1-\exp(-x^2)=\mathrm{erf}(\sqrt{x^2})$. To derive the PDF from its CDF expressed with the error function, it requires us to differentiate the CDF with respect to $\vartheta$, i.e.,
\begin{align}
	f_{L_{PL}}(\vartheta ) =-\frac{\mathrm{d}}{\mathrm{d}\vartheta}\left[\mathrm{erf}\left(\sqrt{\frac{-\ln\vartheta}{2G_0\sigma_p^2}}\right)\right].
\end{align}
By defining $g(\vartheta) = \sqrt{\frac{-\ln\vartheta}{2G_0\sigma_p^2}}$, we have
\begin{align}\nonumber
	f_{L_{PL}}(\vartheta ) =& -\frac{\mathrm{d}}{\mathrm{d}\vartheta}[\mathrm{erf}(g(\vartheta))]\\\nonumber
	=&-\frac{\mathrm{d}}{\mathrm{d}g(\vartheta)}[\mathrm{erf}(g(\vartheta)]\cdot\frac{\mathrm{d}g(\vartheta)}{\mathrm{d}\vartheta}\\\nonumber
	\overset{(a)}{=}&-\frac{2}{\sqrt{\pi}}\exp(-g(\vartheta)^2)\cdot\frac{1}{2}\cdot\frac{1}{\sqrt{\frac{-\ln\vartheta}{2G_0\sigma_p^2}}}\cdot\left(-\frac{1}{\vartheta}\right)\\
	=&\frac{\sigma_p\sqrt{G_0}\vartheta^{\frac{1}{2G_0\sigma_p^2}-1}}{\sqrt{-\pi\ln\vartheta}}.
\end{align}
This completes the proof.
\end{proof}

By Lemma~\ref{lem: 1}, the outage probability in \eqref{eq: outage1} can be rewritten as an integral that explicitly depends on the physical parameters of the laser ISL:
\begin{align}\label{eq: outage prob}
	P_{\text{out}}^{i,j}({\sf SNR}) = \int_{0}^{\Gamma_0^{i,j}}f_{L_{PL}}(\vartheta )\mathrm{d}\vartheta,
\end{align}
where $\Gamma^{i,j}_0 = \frac{\sigma^2{\sf SNR}_{\text{th}}}{P_T^{i}\eta_S G_T^i G_R^jL_{PS}^{i,j}}$.

To compute the overall outage probability of an aggregation path, which contains multi-hop laser ISLs, we first compute its success probability as
\begin{align}\label{eq: problem3}
	P_{\text{succ}}(\mathcal{B}(m,u))=\prod_{i,j\in\mathcal{V}'} [1-p_{\text{out}}^{i,j}(m,u)].
\end{align}
The corresponding outage probability is then given by:
$P_{\text{out}}(\mathcal{B}(m,u)) = 1 - P_{\text{succ}}(\mathcal{B}(m,u))$.
By taking the logarithm of both sides of \eqref{eq: problem3}, we obtain an equivalent expression, i.e.,
\begin{align}\label{eq: problem31}
	\log P_{\text{succ}}(\mathcal{B}(m,u))=\sum_{i,j\in\mathcal{V}'} \log [1-p_{\text{out}}^{i,j}(m,u)].
\end{align}
Therefore, minimizing the outage probability of the aggregation path in each time frame to address the uncertainty is equivalent to minimizing $-\log P_{\text{succ}}(\mathcal{B}(m,u))$, which yields the reformulated optimization problem:
\begin{align}\label{eq: problem32}
	\mathscr{P}_{3}:~\operatorname{minimize}~ & \sum_{i,j\in\mathcal{V}'} \log \left[\frac{1}{1-p_{\text{out}}^{i,j}(m,u)}\right].
\end{align}
The key observation of \eqref{eq: problem32} is that it shares the same mathematical structure as \eqref{eq: problem2}, both aiming to minimize the summation of certain edge weights over a directed sub-branching.
As a result, to balance the energy efficiency and the outage probability of the routing problem, the objective function can be formulated as the linear combination of $\mathscr{P}_{2}$ and $\mathscr{P}_{3}$, i.e.,
\begin{align}\label{eq: problem4}
	\mathscr{P}_{4}:~\operatorname{minimize} ~& \rho W_u(\mathcal{B}(m,u))-(1-\rho)\log P_{\text{succ}}(\mathcal{B}(m,u)).
\end{align}
It is worth noting $\mathscr{P}_{4}$ can be directly solved by adjusting the weight of the directed graph to $w^{r}_{i,j} = \rho w_{i,j} + (1-\rho)\log \frac{1}{1-P_{\text{out}}^{i,j}}$ using Algorithm \ref{alg: chu-liu}.


Finally, we consider the Licklider Transmission Protocol (LTP), which is commonly used in deep space communications where link disruptions are frequent \cite{yu2015performance,yang2023study}. LTP sends data packets in segments, similar to the method employed in this paper where model data packets are divided into 
$ U $ blocks for transmission \eqref{eq: weight frame}.
Once a transmission failure occurs, a retransmission request can be sent to initiate a retransmission to ensure the integrity of the model transmission.

\subsection{Convergence Results}
To analyze convergence performance of the proposed hierarchical learning algorithm, we have the following assumptions and theorem.
		\begin{ass}\label{ass: lower}
	The objective function $f$ is lower bounded by a scalar $f^{\inf}$, i.e., $f(\bm{x})\geq f^{\inf}>-\infty$.
\end{ass}

\begin{ass}[$L$-Smoothness] \label{ass: smooth}
	The loss function of each device $i$ in cluster $j$, each cluster $j$, and all clusters, i.e., $f_{i,j}(\cdot)$, $ f_{j}(\cdot) $, and $ f(\cdot) $ are all $L$-smooth. For all $\bm{x}$ and $\bm{y}$, we have
	\begin{equation*}
		\|\nabla f_{i,j}(\bm{x})-\nabla f_{i,j}(\bm{y})\|_2 \leq L \|\bm{x} - \bm{y}\|,
	\end{equation*}
	and 
	$$f(\bm{y})\leq f(\bm{x}) + \left\langle \nabla f(\bm{x}),\bm{y-x}\right\rangle+\frac{L}{2}\left\|\bm{y-x}\right\|_2^2.$$
\end{ass}

\begin{ass}[Unbiased Gradient and Bounded Variance] \label{ass: gradient variance}
	Each device has an unbiased stochastic gradient estimation with bounded variance, i.e., $$\mathbb{E}[\tilde{\nabla} f_{i,j}(\bm{x}) ] = \nabla f_{i,j}(\bm{x}),$$$$\mathbb{E} \left[ \| \tilde{\nabla} f_{i,j}(\bm{x}) - \nabla f_{i,j}(\bm{x}) \|_2^2 \right] \leq \sigma^2 .$$

\end{ass}

\begin{ass}[Bounded Gradient Dissimilarity]\label{ass: inter-orbit dissimilarity}
	There exists constants $\alpha \geq 1$ and $\beta\geq 0$ such that
	\begin{align}\label{eq: common bound}
		\frac{1}{J}\sum_{n=1}^{J}\|\nabla   f_n(\bm x)\|_2^2\leq \beta+\alpha\|\nabla f(\bm x)\|_2^2.
	\end{align}
\end{ass}

Assumptions \ref{ass: lower}, \ref{ass: smooth}, and \ref{ass: gradient variance} are commonly used in stochastic optimization for the objective function \cite{bottou2018optimization}.
Assumption \ref{ass: inter-orbit dissimilarity} is adopted to characterize the non-IID extent of the local data distribution \cite{zhang2020fedpd}. Similar assumptions
have also been made in the literature \cite{yu2019on,fang2022communication} for the convergence analysis under the non-i.i.d. setting. 

\begin{thm}[Error Bound of Hierarchical Learning Under Nonconvexity]\label{thm: 1}
	Let Assumptions \ref{ass: lower}, \ref{ass: smooth}, \ref{ass: gradient variance}, and \ref{ass: inter-orbit dissimilarity} hold. When the learning rate is bounded as
	\begin{align}\nonumber
		\eta = \frac{1}{L}\sqrt{\frac{J}{ET}}\leq\min\left\{\frac{1}{2LE},\frac{1}{L\sqrt{2E(E-1)(2\alpha+1)}}\right\},
	\end{align}
	the minimal gradient norm of the global loss function after $T$ communication rounds is bounded by
	\begin{align}\nonumber
		\min_{t\in[T]}\mathbb{E}\left[\|\nabla f(\bm{x}^t)\|_2^2\right]\leq&\mathcal{O}\left(\frac{1+\bar{\sigma}^2}{\sqrt{JET}}\right)\\+&\mathcal{O}\left(\frac{J(E-1)(\sigma^2+2E\beta)}{ET}\right),
	\end{align}
	where $\mathcal{O}(\cdot)$ omits the constants and $\bar{\sigma}^2 = J\sum_{n=1}^{J}\sigma^2$.  
\end{thm}
\begin{proof}
	It is worth noting that the aggregation is based on the average of all devices' local information.
	Based on the updating rules, i.e., 
	\begin{align*}
		\bm{x}^{t+1}:=\bm{x}^{t}-\eta\sum_{n=1}^{J}\sum_{e=0}^{E-1} \lambda_{n}\tilde{\nabla}f_{n}(\bm{x}_{n}^{t,e}),
	\end{align*}
	and the authors previous results in \cite[Supplementary Material]{zhu2024over}, the average
	norm of global gradients after $T$ communication rounds can be bounded by
	\begin{align}\label{eq: average1}\nonumber
		&\frac{1}{T}\sum_{t=0}^{T}\mathbb{E}\left[\|\nabla f(\bm{x}^t)\|_2^2\right]\\\nonumber
		\leq&\frac{4[f(\bm{x}^{0})-f^{\inf}]}{\eta ET}+4\eta L\sigma^2\sum_{n=1}^{J}\lambda_{n}^2+3\eta^2L^2\sigma^2(E-1)\\
		+&6\eta^2L^2\beta E(E-1).
	\end{align}
	The learning rate should satisfy $\eta\leq\min\{\frac{1}{2LE},\frac{1}{L\sqrt{2E(E-1)(2\alpha+1)}}\}$.
	In addition, \eqref{eq: average1} can be further optimized by setting
	$\eta = \frac{1}{L}\sqrt{\frac{J}{ET}}$, which completes the proof.
\end{proof}

From Theorem \ref{thm: 1}, we can conclude that the convergence rate of the hierarchical learning algorithm in this paper is decided by $\mathcal{O}\left(\frac{1+\bar{\sigma}^2}{\sqrt{JET}}\right)$, which presents a linear speedup caused by $E$ rounds of local updates. An additional error is caused by data heterogeneity and multiple local updates, which means the proposed algorithm has an error gap compared with centralized training.


\section{Simulations}\label{sec: simulations}
\subsection{Experiment Setup}\label{sec: setup}
\subsubsection{Network Configurations}
To verify the performance of the proposed algorithms, we consider a $80/4/1$ Walker-Star constellation system and a $80/4/1$ Walker-Delta constellation system, which consists of $N=4$ orbit planes and $S_n = 20$ satellites evenly distributed in each orbit for default settings. Three GEO satellites are evenly distributed in the Equatorial plane orbit.
In order to characterize the impact of the network size, we also consider a $240/8/1$ Walker-Star constellation system for comparison.
We also consider a $800/20/1$ large Walker-Star constellation system to simulate the real-world network scale.
The inclination of the Walker-Delta constellation is 45 degree and that of Walker-Star constellation is 99.5 degree. The altitude of the Walker-Delta orbits is 500 km and that of Walker-Star constellation is 700 km.
Fig. \ref{fig: simulation system} illustrates the virtual space-ground integrated system constructed for our simulations, encompassing terrestrial devices, different LEO satellite constellations, and GEO satellites. Simulation system with $240/8/1$ and $800/20/1$ Walker-Star constellation can be extended from Fig. \ref{fig: walker-star}.

\begin{figure}[]
	\begin{subfigure}{0.24\textwidth}
		\centering
		\includegraphics[width=1\linewidth]{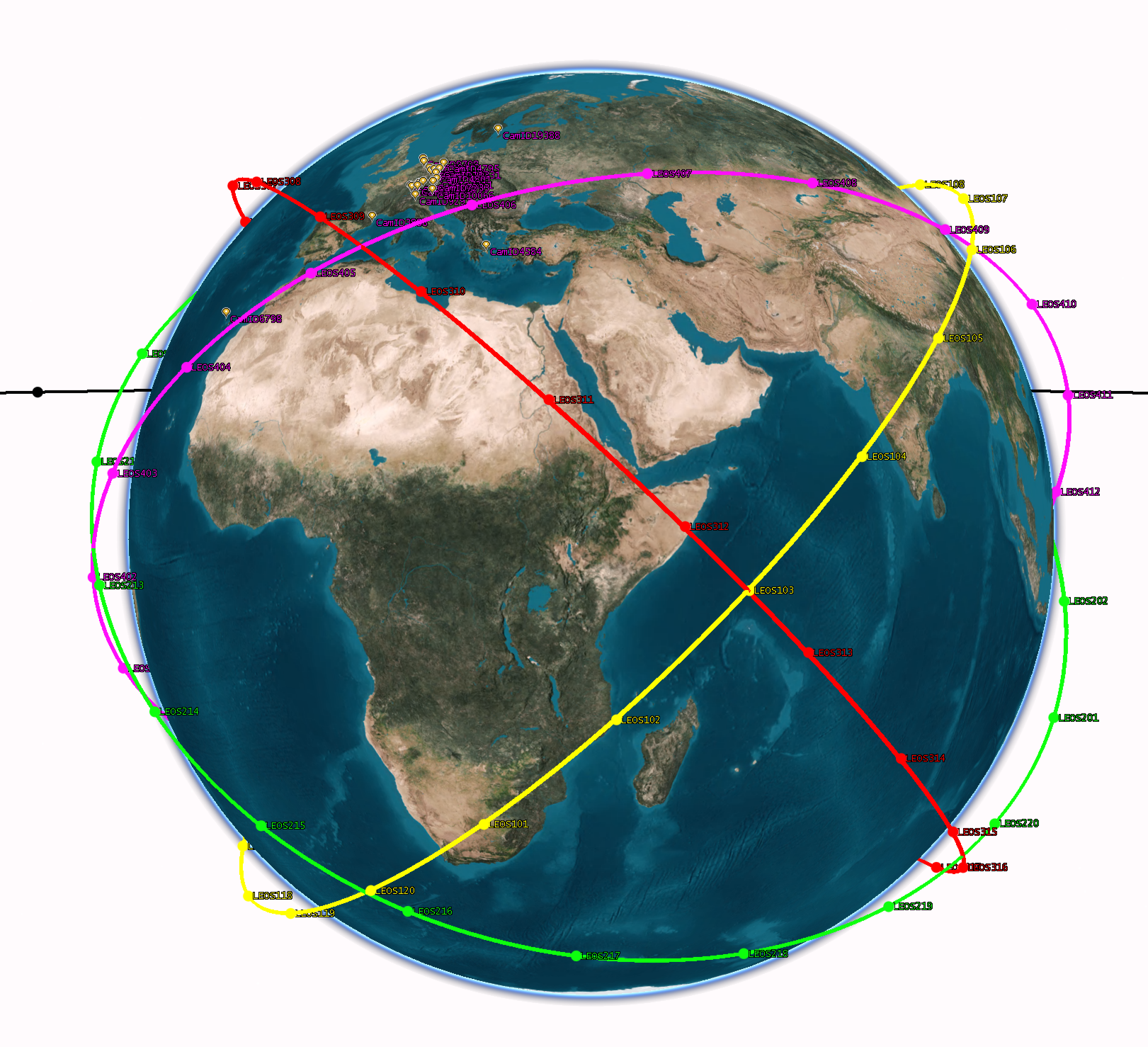}
		\caption{Simulation system with a 80/4/1 Walker-Delta constellation.}
		\label{fig: walker-delta}
	\end{subfigure}
	\hfill
	\begin{subfigure}{0.24\textwidth}
		\centering
		\includegraphics[width=1\linewidth]{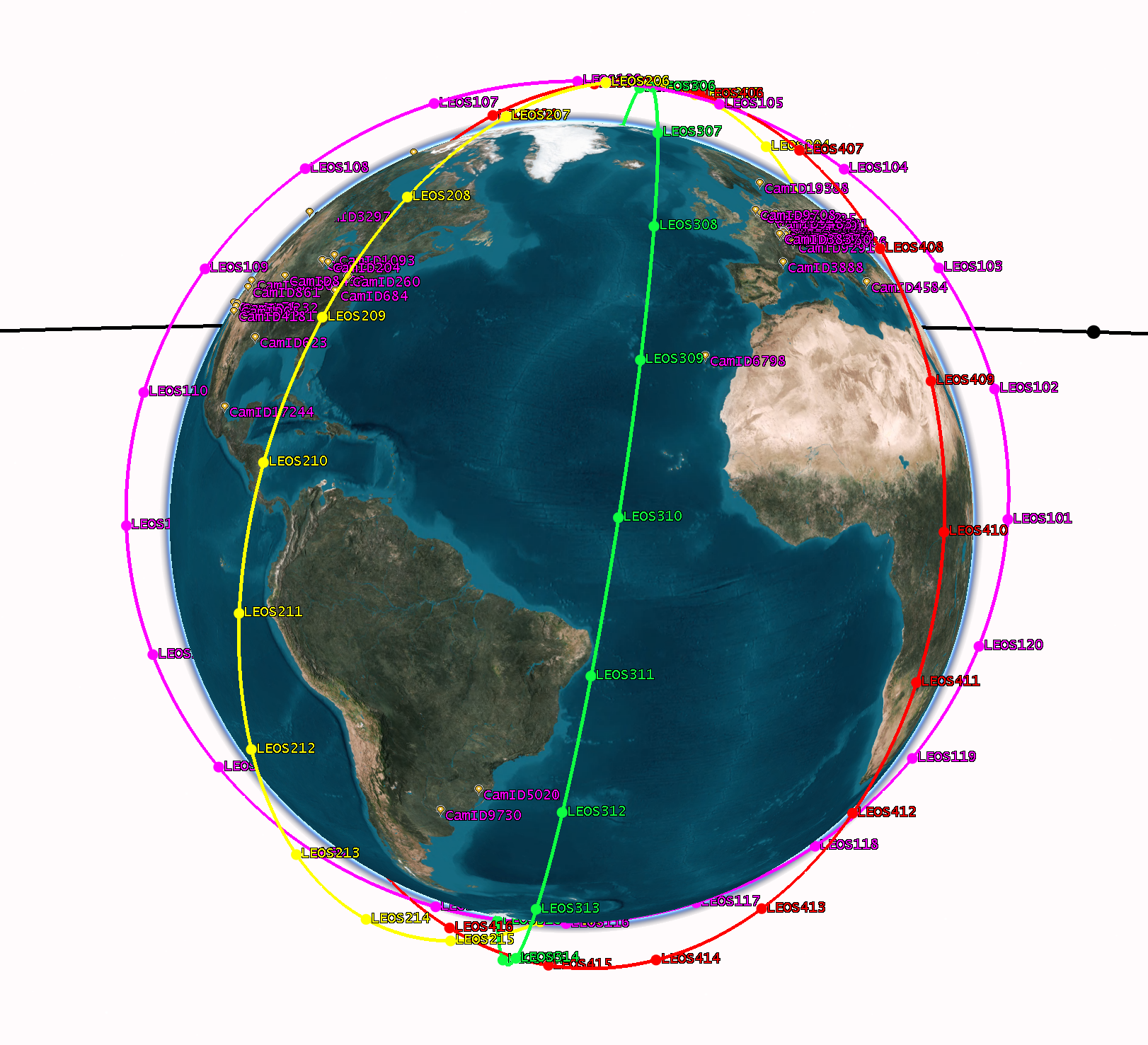}
		\caption{Simulation system with a 80/4/1 Walker-Star constellation.}
		\label{fig: walker-star}
	\end{subfigure}
	\caption{Simulation system with different LEO constellations.}
	\label{fig: simulation system}
	\vspace{-0.5cm}
\end{figure}

\begin{figure*}[!]
	\begin{subfigure}{0.32\textwidth}
		\centering
		\includegraphics[width=1\linewidth]{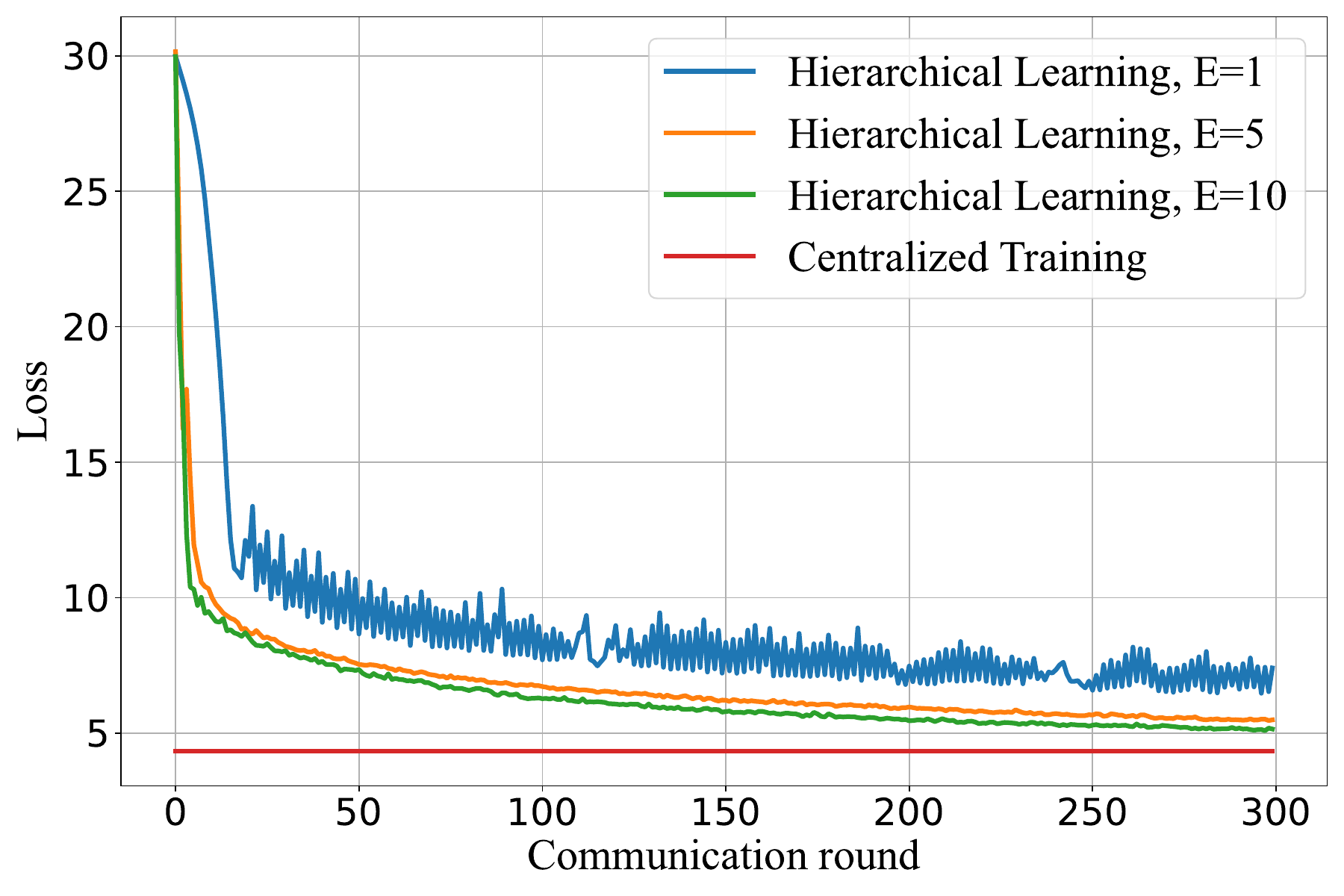}
		\caption{Training loss versus number of communication rounds.}
		\label{fig: loss}
	\end{subfigure}
	\hfill
	\begin{subfigure}{0.32\textwidth}
		\centering
		\includegraphics[width=1\linewidth]{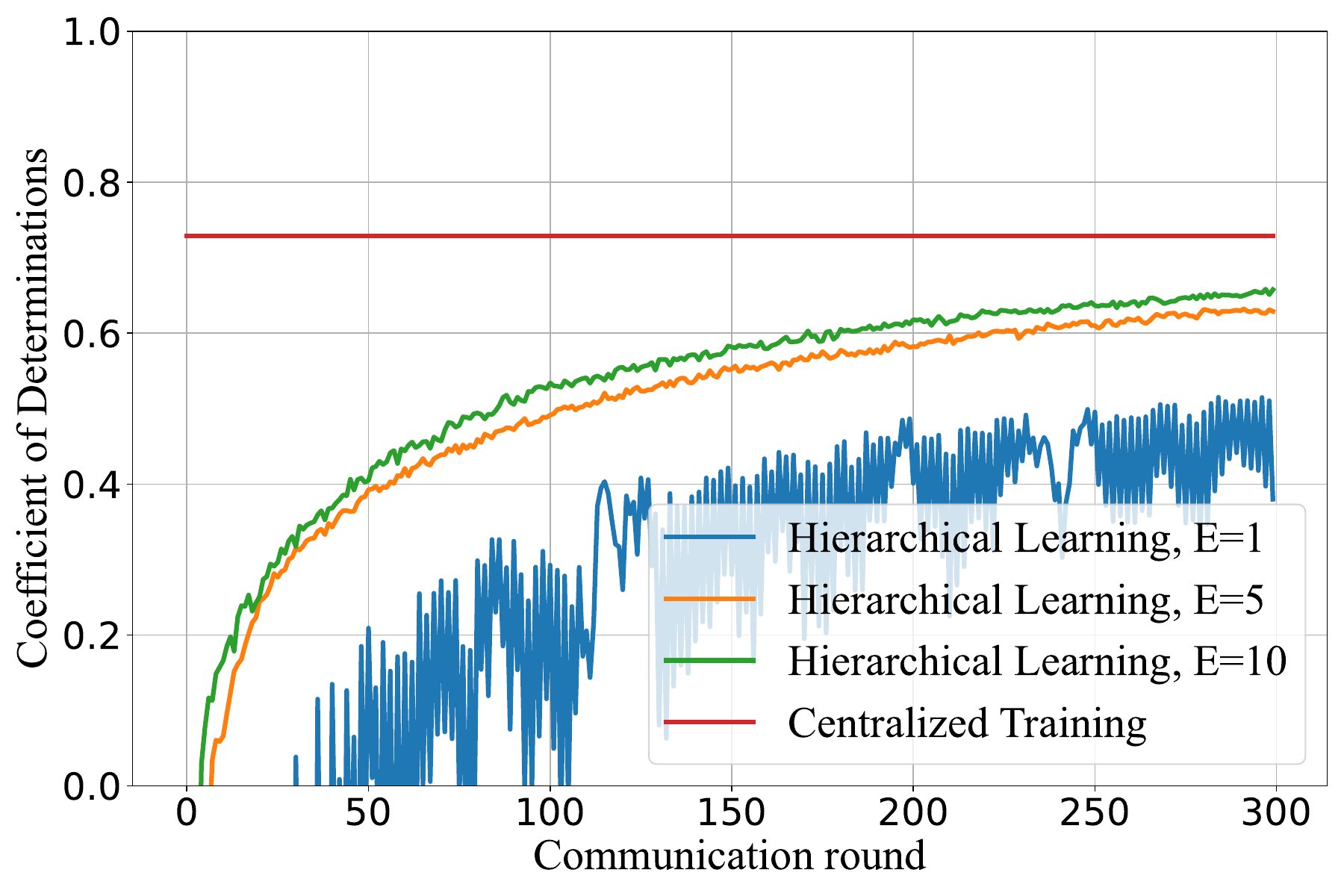}
		\caption{$R^2$ versus number of communication rounds.}
		\label{fig: r2}
	\end{subfigure}
	\hfill
	\begin{subfigure}{0.32\textwidth}
		\centering
		\includegraphics[width=1\linewidth]{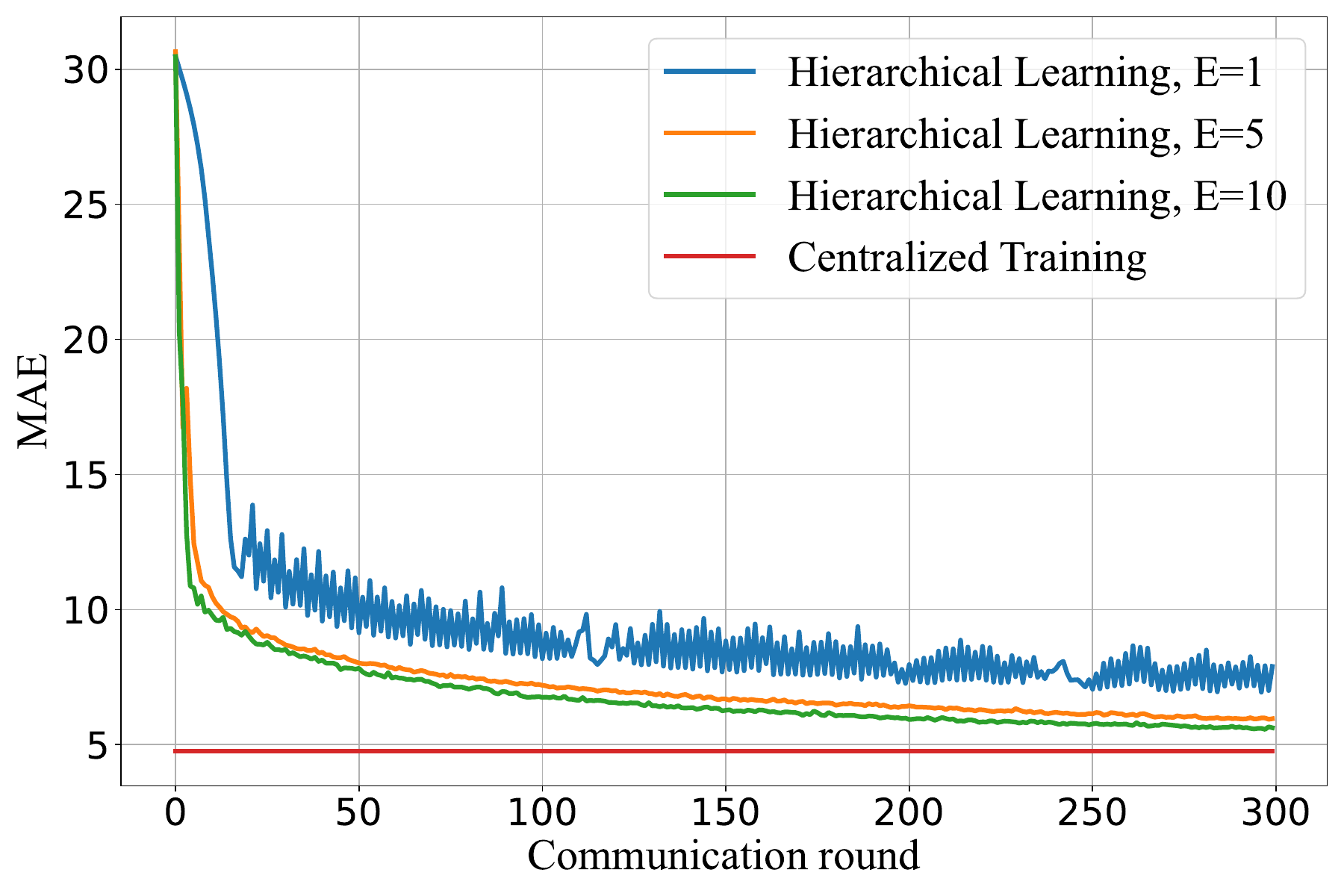}
		\caption{Testing MAE versus number of communication rounds.}
		\label{fig: mae}
	\end{subfigure}
	\caption{The learning performances of hierarchical learning versus centralized training.}
	\label{fig: learning performance}
	\vspace{-0.5cm}
\end{figure*}

\subsubsection{Learning Tasks}
Retrieving atmospheric environmental parameters from digital images is a new approach for horizontal environmental monitoring.
We consider the Sub-SkyFinder (SSF) dataset \cite{su2022retrieval}, which contains 24,328 webcam images captured by 41 different webcams in different locations. The digital images in the SSF dataset are captured during the years 2011–2014. The reference values of the SSF datasets, i.e., visibility, relative humidity, and temperature, are available from the weather stations that are nearby the locations where the webcams installed. The learning task is a regression task, where the input data are webcam images with $160\times120$ pixels and the labels are the values of visibility, relative humidity, and temperature corresponding to each image. $J=41$ webcams serve as terrestrial IoT devices and each IoT device trains an AlexNet \cite{krizhevsky2012imagenet} as the local model for the local learning task.
Each webcam only utilizes local collected data for training its local model.
Specifically, we use a pre-trained AlexNet model and modify the final fully connected layer in the classifier section from the original 1000 output nodes to 3 output nodes to fit our regression requirements. The loss function is $\mathsf{SmoothL1Loss}$, which is given by
\begin{align}
\ell(x) = 
\begin{cases} 
	\frac{1}{2} x^2 & \text{if } |x| < 1, \\
	|x| - \frac{1}{2} & \text{otherwise}.
\end{cases}
\end{align}
The batch-size is set to be 32. The learning rate is set to be $1 \times 10^{-3}$.
The number of communication rounds $T$ and the number of local update $E$ are set to be $300$ and $5$, respectively.
Three metrics are used for analyzing the learning performance, which are the loss, the coefficient of determinations ($R^2$), and the mean absolute error (MAE).
We compare the proposed hierarchical learning $E=\{1, 5, 10\}$ with geographical located devices with the centralized training case, where all the data are collected and training on the cloud server to verify the convergence results.
We run all the experiments on CUDA version 11.7, PyTorch version 2.0.0, and Python version 3.9. 
The terrestrial webcams are also marked in Fig. \ref{fig: simulation system}.

\begin{table}[t]
	\centering
	\caption{Parameters Settings}
	\renewcommand\arraystretch{1.2}
	\footnotesize
	\begin{tabular}{l|l}
		\hline
		\textbf{Parameter} & \textbf{Value} \\ \hline
		Constellation & Walker Star / Delta \\
		Number of LEOs ($S$) & 80 \\ 
		Number of orbits ($N$) & 4 \\ 
		Altitude ($h$) & 700 / 500 km \\ 
		Length of time slots ($\Delta p$) & 250 s\\
		Length of time frames ($\Delta \tau$) & 10 s\\
		\hline
		Wavelength ($\lambda$)& 2.14 cm\\
		Laser carrier frequency ($f_c$) & 193 THz \\ 
		System bandwidth ($B$)& 0.02$f_c$\\
		Transmit power ($P_T$) & $\mathrm{Unif}$(0.0316, 5) W \\ 
		System optical efficiency ($\eta_S$) & 0.8 \\ 
		Receiving telescope diameter ($D_R$) & 6 mm \\ 
		Pointing error ($\theta_0$) & 0.01 rad \\
		3-dB Beamwidth ($\theta_{\text{3dB}}$)& 0.1 rad\\
		Boltzmann constant ($k_b$) & 1.38$\times10^{-23}$ J/K \\
		Solar brightness temperature ($T_s$) & 6000 K \\
		System noise temperature ($T_0$) & 1000 K \\
		CMB temperature ($T_{\text{CMB}}$)& 2.725 K \\
		Distribution parameter ($\sigma_p$) & 0.05 \\ 
		SNR threshold (${\sf SNR}_{\text{th}}$) & -110 dB \\ \hline
		Number of devices ($J$) & 41 \\
		Total communication rounds ($T$) & 300 \\
		Number of local updates ($E$) & 5 \\ 
		Batch size & 32 \\ 
		Learning rate $\eta$& 0.001\\
		\hline
	\end{tabular}
	\label{tab: parameters}
\end{table}

\subsubsection{Comparison Algorithms}
We consider the following algorithms which can be applied to solving $\mathscr{P}_2$ and $\mathscr{P}_4$.
\begin{itemize}
	\item \textbf{TAEER:} The proposed algorithm presented in Algorithm \ref{alg: chu-liu}. The root node $r$ is chosen from the terminals and is responsible for relaying the aggregation model to the GEO satellite.
	\item \textbf{D-Merge:} We also propose a simple solution by first finding the shortest distance paths from all the terminals to the root then merging all the paths to form a spanning tree for approximation. This method is a $|\mathcal{V}'|$-approximation of the \textit{DST} problem. Also, the root node $r$ is chosen from the terminals and is responsible for relaying the aggregation model to the GEO satellite.
	\item \textbf{Orbit-Greedy:} Inter-orbit ISLs are not considered in this case. Firstly, the minimum arc that containing all terminals in one orbit is established for each orbit. Then, it randomly selects a root node in the orbital minimum arc to which the rest of the terminals transmit their collected models. All root nodes transmit the final information to the GEO satellite.
\end{itemize}

\subsubsection{Experimental Platforms}
We use the STK version 11.6 tool to model the SGINs. Specifically, the $\mathrm{Chain}$ and $\mathrm{Constellation}$ objects are used for modeling both the ground-space connection and space topology. 
Specifically, $\mathrm{Chain}$ objects that connects ground devices and the LEO constellations, i.e., the $\mathrm{Constellation}$ objects, are utilized to determine the terminal satellites receiving local models. Matlab-STK connectors is utilized to extract range information between LEO satellites for space topology construction.
The scenario time period is set to be from 25 Jun 2024 02:00:00.000 UTCG to 26 Jun 2024 02:00:00.000 UTCG.

The parameter settings in this paper are summarized in Table \ref{tab: parameters}, where the selection of parameters is primarily based on \cite{zhai2023fedleo,polishuk2004optimization, nie2021channel, arnon2005performance, wang2007topological}. The simulation code of this paper under a Walker Star constellation available at https://github.com/Jingyang-Zhu/HierLC as a demonstration.


\subsection{Performance Evaluation}
In this subsection, we evaluate the performance of the proposed algorithm from the following aspects.
\subsubsection{Learning Performance versus Communication Rounds}
The training loss, test $R^2$, and MAE versus communication rounds of the environmental parameter retrieval are shown in Fig. \ref{fig: learning performance}.
As observed in Fig. \ref{fig: learning performance}, our proposed hierarchical learning architecture progressively approaches the optimal performance of centralized training as the number of communication rounds increases. The performance gap is attributed to the non-independent and identically distributed (non-IID) nature of the data generated at terrestrial devices in different geographic locations. Given that the hierarchical learning architecture does not require uploading end data to the cloud server, it effectively protects image privacy for the environmental monitoring applications.
{In addition, to demonstrate the impact of the number of local updates on the hierarchical learning architecture, as shown in Fig. \ref{fig: learning performance}, the learning performance and convergence speed under $E=10$ outperform those with $E=5$ and $E=1$, which aligns with the conclusions of our convergence analysis.}

\begin{figure}[!]
	\begin{subfigure}{0.49\textwidth}
		\centering
		\includegraphics[width=1\linewidth]{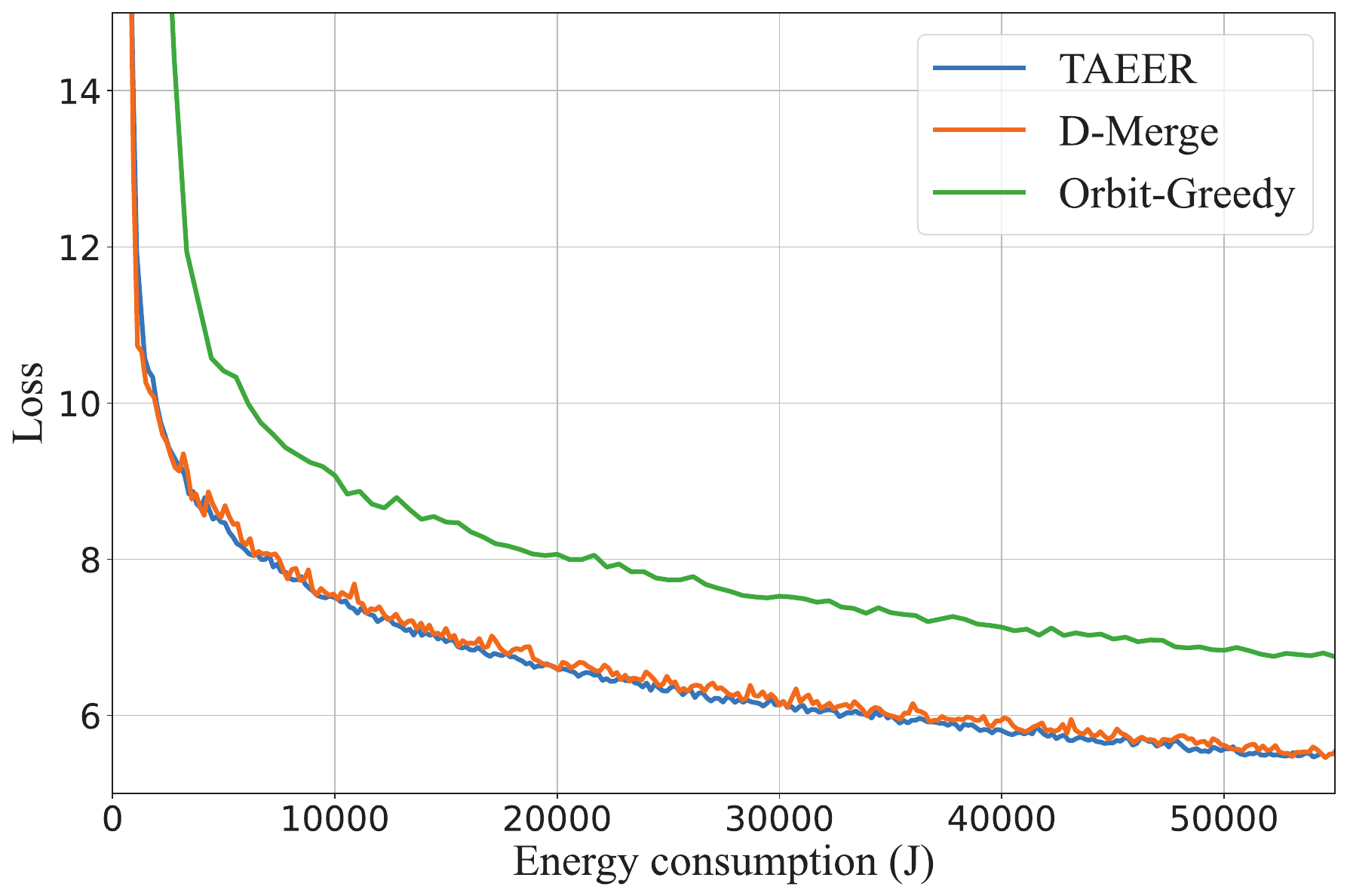}
		\caption{Training loss in a 80/4/1 Walker-Delta constellation.}
		\label{fig: en walker-delta}
	\end{subfigure}
	\hfill
	\begin{subfigure}{0.49\textwidth}
		\centering
		\includegraphics[width=1\linewidth]{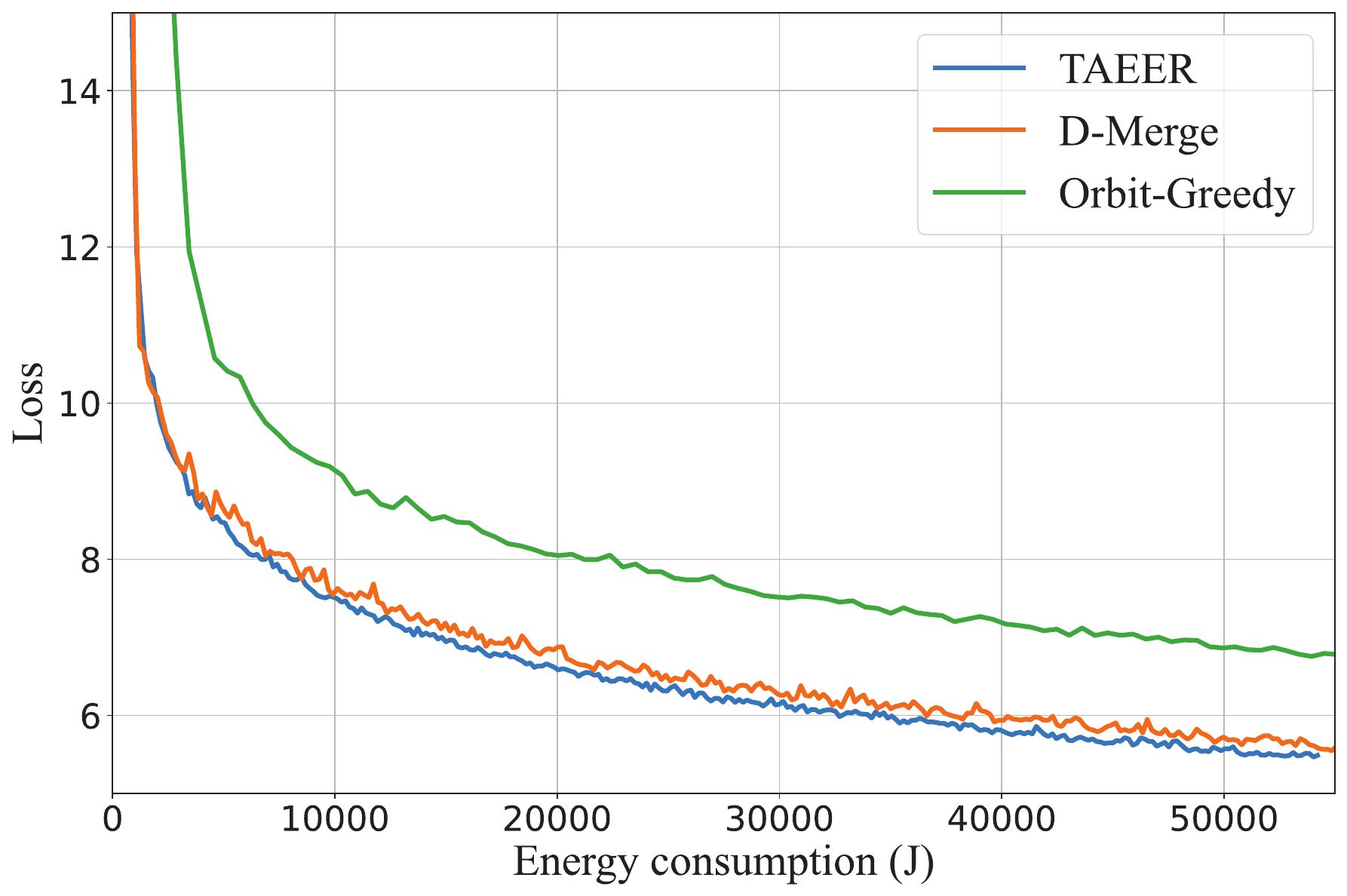}
		\caption{Training loss in a 80/4/1 Walker-Star constellation.}
		\label{fig: en walker-star}
	\end{subfigure}
	\caption{Learning performance versus system energy consumption.}
	\label{fig: energyvsloss}
	\vspace{-0.5cm}
\end{figure}

\begin{table*}[t]
	\centering
	\caption{Simulation results under different schemes and network configurations}
	\footnotesize
	\renewcommand\arraystretch{1.2}
	\begin{tabular}{c|c|ccc|ccc}
		\hline
		&&\multicolumn{3}{c|}{$\rho=1$}&\multicolumn{3}{c}{$\rho=0.1$}\\\hline
		Metric & Constellation & TAEER & D-Merge & Orbit-Greedy & TAEER & D-Merge & Orbit-Greedy\\\hline
		\multirow{4}{*}{Average energy consumption per time slot (J)}
		& 80/4/1 Walker-Delta & 180.68 & 187.50 & 555.73 & 183.90 & 221.41 & 556.38\\	
		& 80/4/1 Walker-Star & 202.38 & 208.42 & 573.90 & 211.49 & 276.25 & 575.79\\
		& 240/8/1 Walker-Star & 151.20 & 153.87 & 552.79 & 151.97 & 161.24 & 552.87\\
		& 800/20/1 Walker-Star&160.21&169.86&542.78&160.51&182.03&542.78\\
		\hline
		\multirow{3}{*}{Average outage probability for each ISL (\%)}	
		& 80/4/1 Walker-Delta & - & - & - & 6.71 & 9.70 & 5.37\\
		& 80/4/1 Walker-Star & - & - & - & 8.97 & 15.3 & 7.50\\
		& 240/8/1 Walker-Star & - & - & - & 1.56 & 2.06 & 1.47\\
		&800/20/1 Walker-Star&-&-&-&0.75&1.07&0.92\\
		\hline
	\end{tabular}
	\label{tab:results}
\end{table*}

\subsubsection{Overall Results Analysis}
Table \ref{tab:results} illustrates the overall simulation results for the TAEER problem. The performance metrics are average energy consumption per time slot (J) and average outage probability for each ISL (\%) after $T=300$ communication rounds.
$\rho=1$ represents the case of error-free ISL and $\rho = 0.1$ represents the case for solving $\mathscr{P}_4$, which takes outage probability into account.
As clearly demonstrated in Table \ref{tab:results}, under the default settings, our proposed TAEER algorithm consistently achieves the lowest energy consumption, followed by the D-Merge method, while the Orbit-Greedy method exhibits the highest energy consumption. With respect to outage probability, Orbit-Greedy yields the lowest value, followed by TAEER, with D-Merge showing the highest. This disparity arises because the Orbit-Greedy method exclusively utilizes stable intra-orbit ISLs, leading to multiple high-energy consumption links transmitting to GEO satellites within the network. In contrast, both TAEER and D-Merge account for unstable inter-orbit ISLs.
By solving the topology-aware \textit{DST} problem, TAEER algorithm and D-Merge achieve better overall performance. For instance, in the simulation system with Walker-Delta constellation, TAEER demonstrates a significant energy reduction, up to 3.64\% and 67.50\% compared with D-Merge and Orbit-Greedy, and a 30.82\% outage probability reduction compared with D-Merge.
We can conclude that the heuristic solution, i.e., TAEER, has the best performance compared to the simple approximation, i.e., D-Merge, and the Orbit-Greedy scheme.

Besides, under default settings, each ISL transmission latency is on the millisecond level, while the length of each time frame is $\Delta\tau = 10$ seconds. Therefore, if an ISL transmission happens to experience an transmission outage within a time frame, it can be quickly retransmitted without affecting the aggregation results. We calculate the time required to solve $\mathscr{P}_2$ with $800/20/1$ configuration, the convergence time of the proposed TAEER algorithm is 0.027 seconds, while that of D-Merge is 0.003 seconds. Both are significantly shorter than $\Delta\tau = 10$ seconds. Therefore, our proposed algorithms can meet the real-time requirements of satellite network applications.
Additionally, the ISL capacity is at the Gbps level, capable of accommodating the transmission of larger-scale models, and it is possible to complete the transmission of an entire model data packet within a single time frame.

From $\rho = 1$ to $\rho = 0.1$ in Table \ref{tab:results}, we can see an increase in the average energy consumption.
This is caused by the extra energy consumed by retransmission protocols in LTP when a transmission outage occurs at an ISL.

\begin{figure}[!]
	\begin{subfigure}{0.49\textwidth}
		\centering
		\includegraphics[width=1\linewidth]{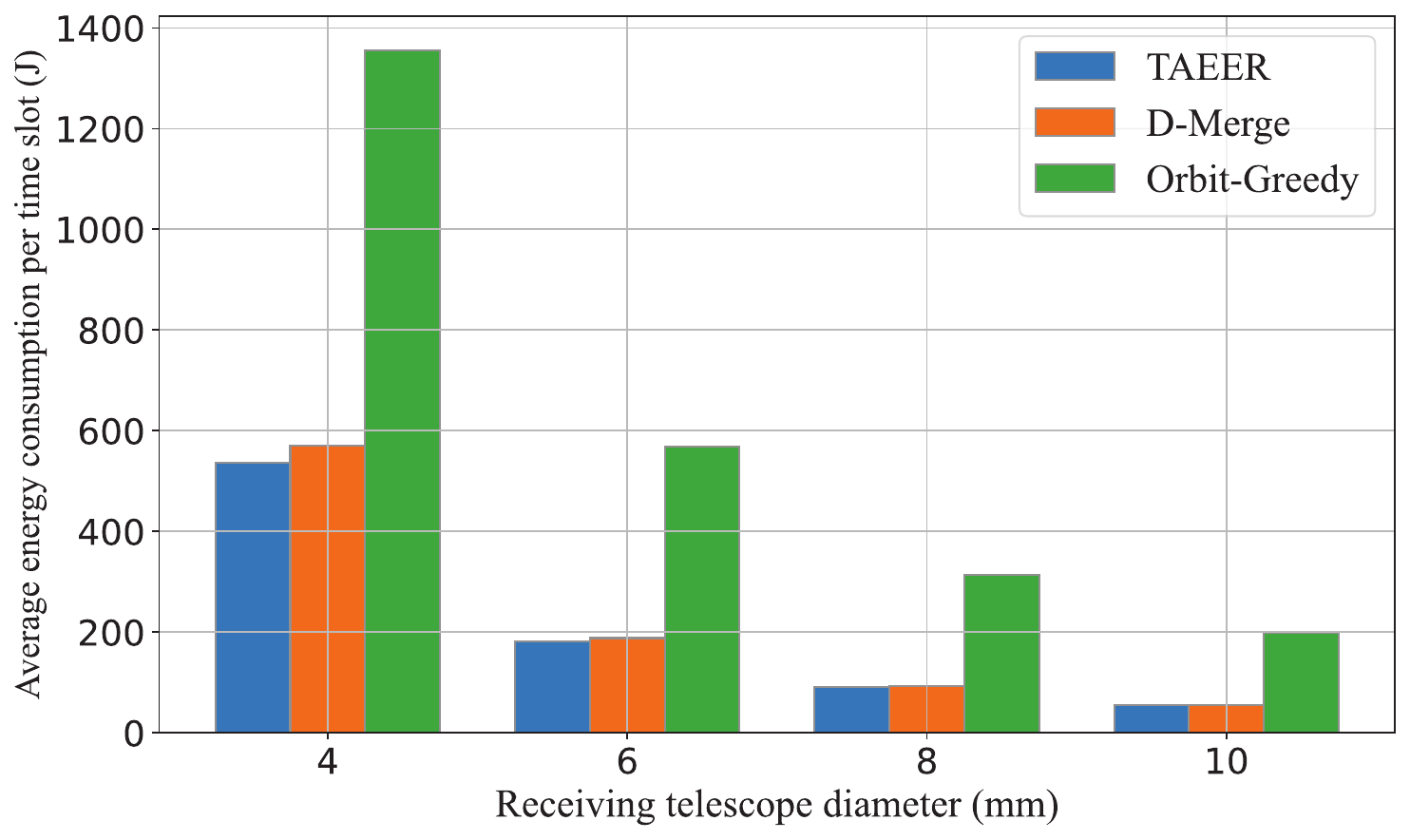}
		\caption{80/4/1 Walker-Delta constellation.}
		\label{fig: ant walker-delta}
	\end{subfigure}
	\hfill
	\begin{subfigure}{0.49\textwidth}
		\centering
		\includegraphics[width=1\linewidth]{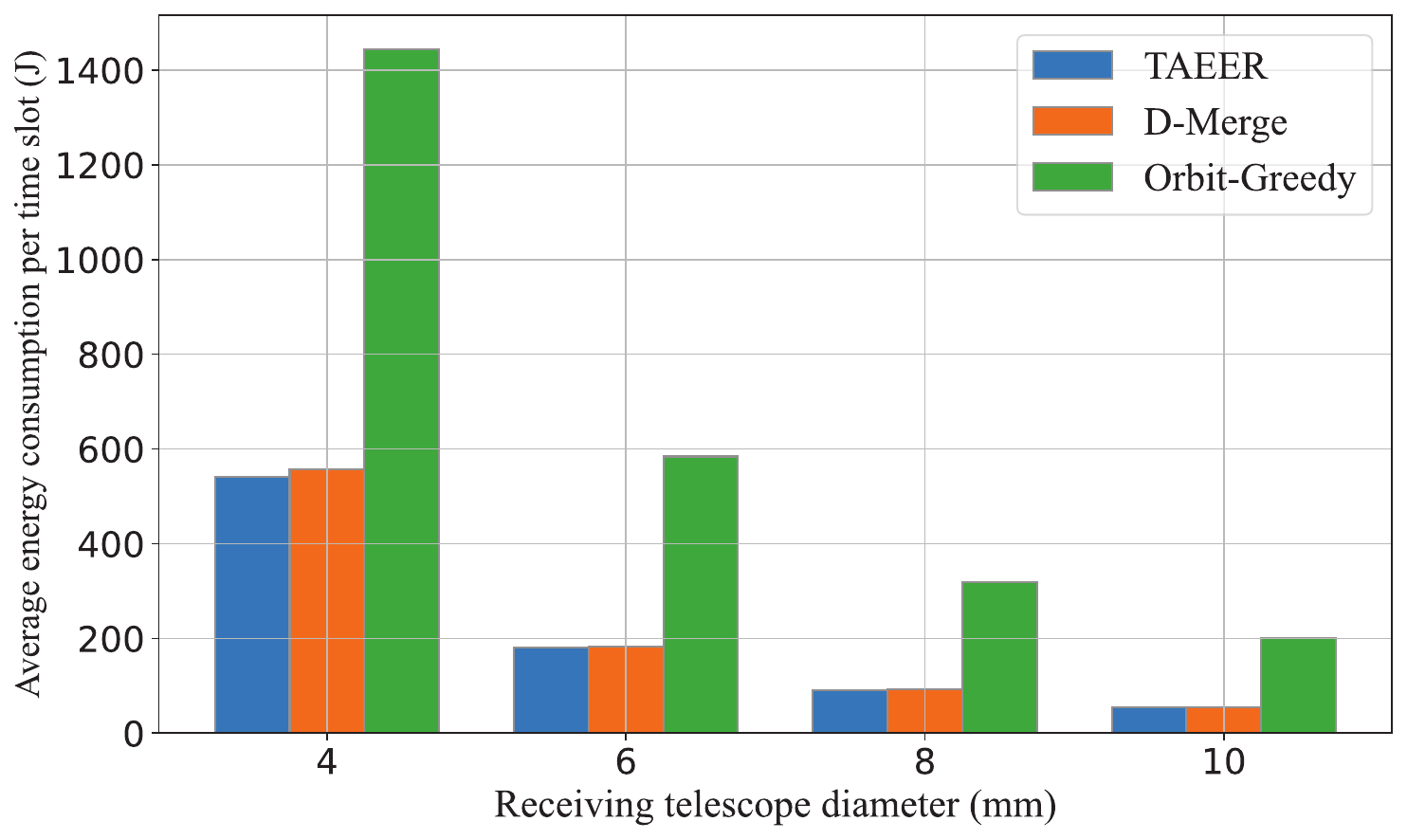}
		\caption{80/4/1 Walker-Star constellation.}
		\label{fig: ant walker-star}
	\end{subfigure}
	\caption{Impact of the receiving telescope diameter on average energy consumption per time slot.}
	\label{fig: ant simulation system}
	\vspace{-0.5cm}
\end{figure}

\subsubsection{Learning Performance versus System Energy Consumption}
The training loss versus energy consumption in different LEO networks of the environmental parameters retrieval task is illustrated in Fig. \ref{fig: energyvsloss}.
To analyze the performance of the proposed TAEER and the D-Merge, we first conclude that D-Merge is a well-established approximation algorithm proposed for solving the \textit{DST} problem, which has the similar performance as the proposed TAEER. While this approach has low complexity, it may lead to suboptimal performance due to insufficient path sharing.
TAEER method constructs subgraphs and fully leverages the benefits of global path sharing and relay nodes to enhance the performance. We shall provide further comparisons between these two schemes in the following subsections.
We can also conclude that the proposed solutions, compared to Orbit-Greedy, which does not consider inter-orbit ISLs, significantly reduce the energy consumption in the LEO satellite network required to achieve the same learning accuracy across two different constellation configurations.

\begin{figure}[!]
	\begin{subfigure}{0.49\textwidth}
		\centering
		\includegraphics[width=1\linewidth]{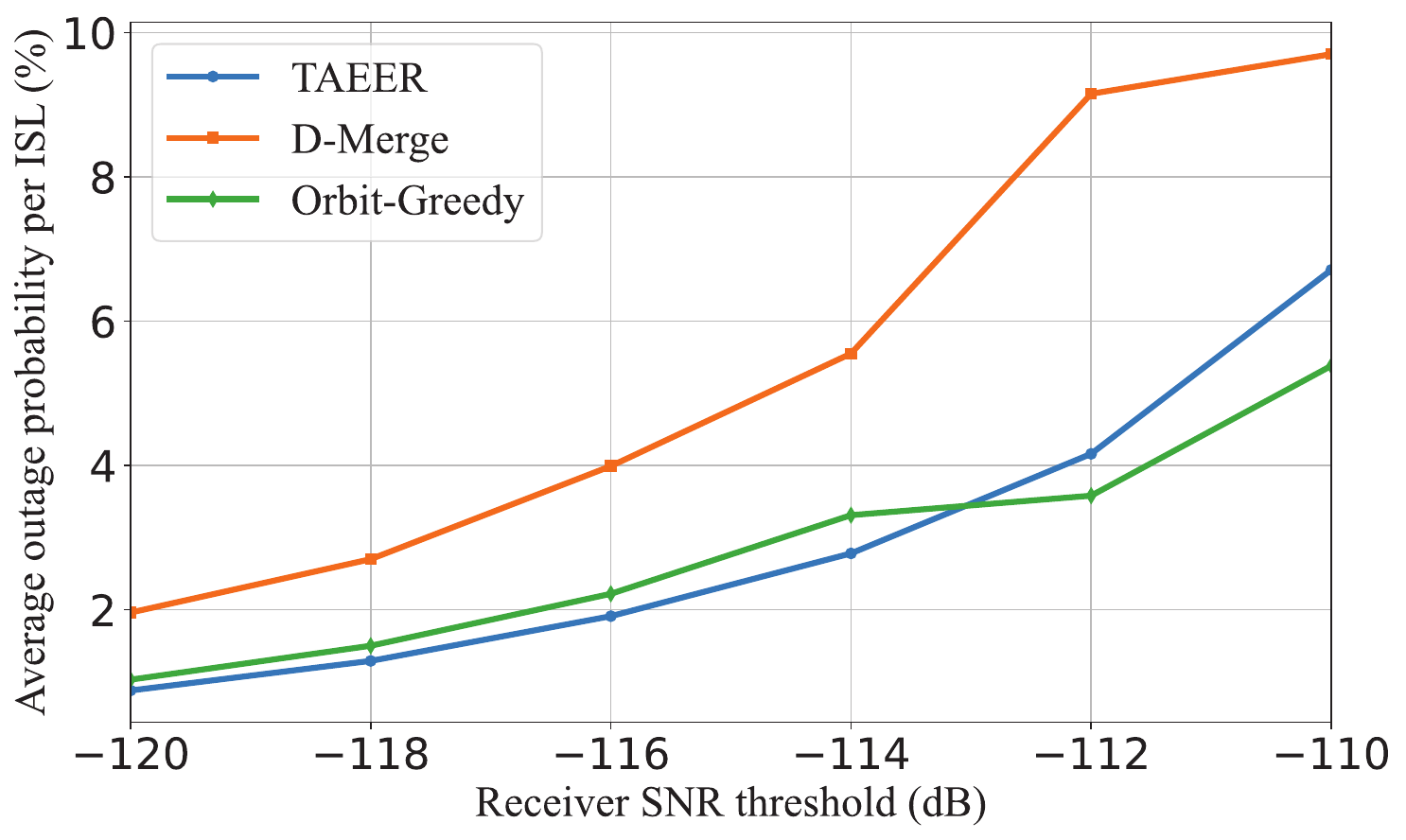}
		\caption{80/4/1 Walker-Delta constellation.}
		\label{fig: SNR walker-delta}
	\end{subfigure}
	\hfill
	\begin{subfigure}{0.49\textwidth}
		\centering
		\includegraphics[width=1\linewidth]{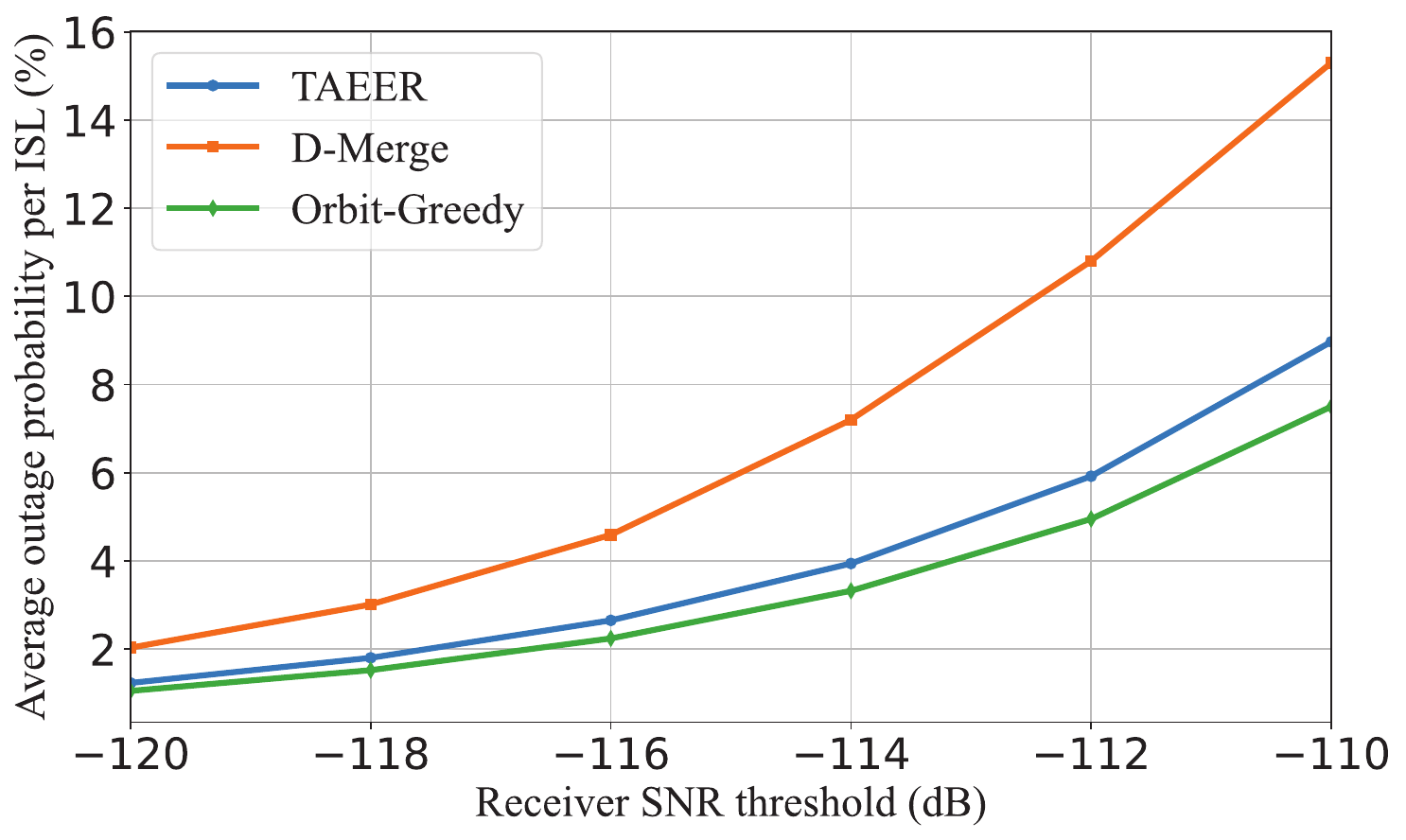}
		\caption{80/4/1 Walker-Star constellation.}
		\label{fig: SNR walker-star}
	\end{subfigure}
	\caption{Impact of the receiver SNR threshold on average outage probability per ISL.}
	\label{fig: SNR simulation system}
	\vspace{-0.5cm}
\end{figure}

\begin{figure*}
	\begin{subfigure}{1\textwidth}
	\centering
	\includegraphics[width=1\linewidth]{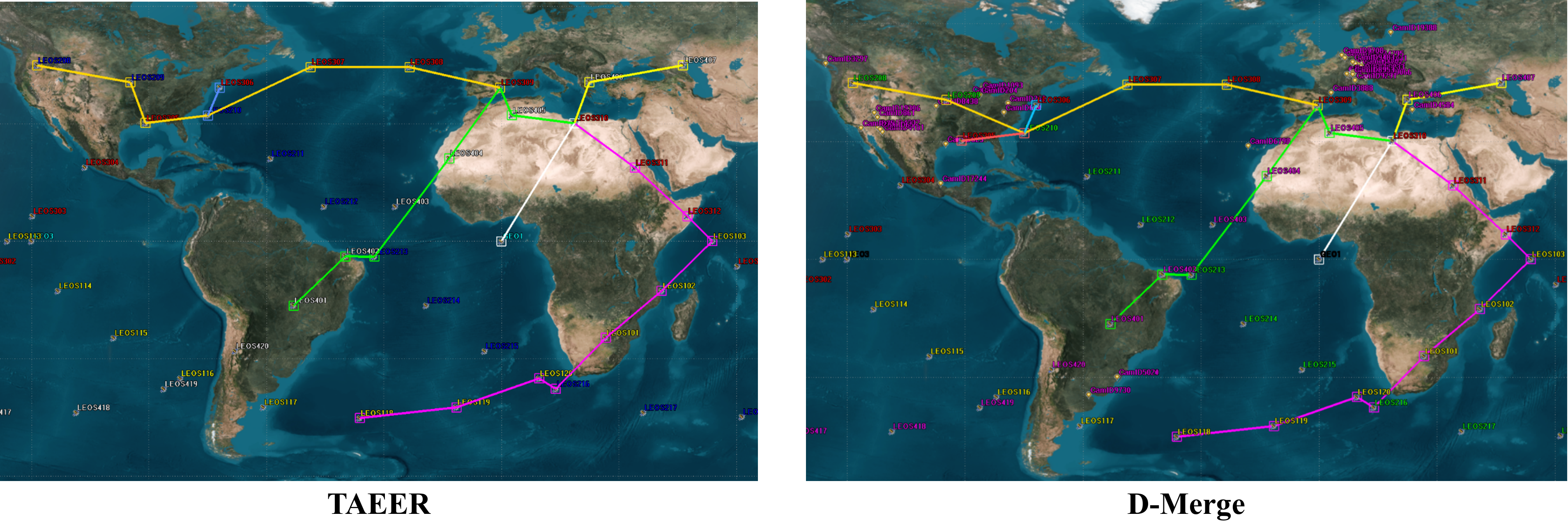}
	\caption{Walker-Delta, time slot 0, snapshot 25.}
	\label{fig: vis delta}
\end{subfigure}
\begin{subfigure}{1\textwidth}
	\centering
	\includegraphics[width=1\linewidth]{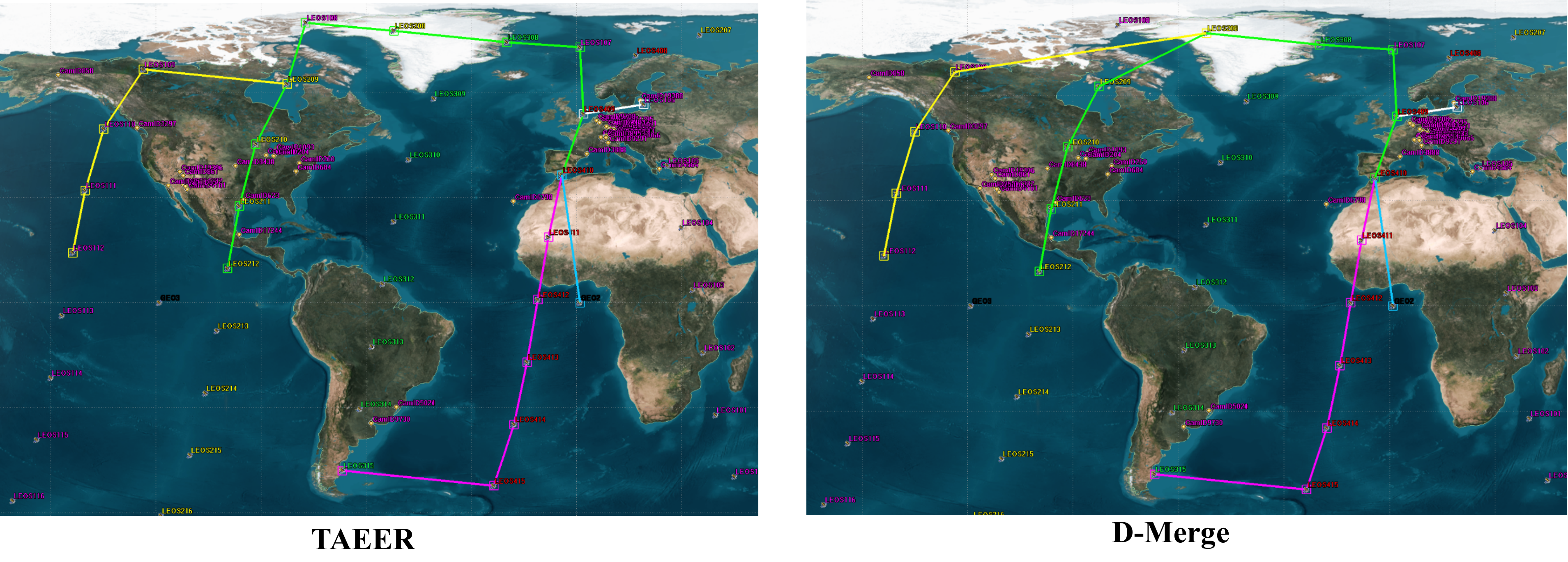}
	\caption{Walker-Star, time slot 18, snapshot 25.}
	\label{fig: vis star}
\end{subfigure}
	\caption{Aggregation paths visualization by solving the \textit{DST} problem in satellite network topology.}
	\label{fig: visualization}
	\vspace{-0.3cm}
\end{figure*}
\subsubsection{Impact of LEO Constellations}
First of all, with the same number of LEO satellites, the system with Walker-Delta constellation exhibits lower average energy consumption and outage probability compared to the Walker-Star constellation due to the closer inter-satellite distances at lower orbital altitudes, resulting in stronger received signal strength.
Secondly, as the number of LEO satellites increases, both the density of satellites within the same orbit and the density of orbital distribution increase. This results in significantly reduced inter-satellite distances within and between orbits, substantially enhancing ISL receiving power. Consequently, the averaged transmission energy required to complete a single global model aggregation is significantly reduced from $80/4/1$ to $240/8/1$ constellation.
We can also observe that as the number of satellites increases further, the averaged transmission energy of the $800/20/1$ constellation is slightly higher compared to the $240/8/1$ constellation due to the more frequent relay switching of ISL hops.
Besides, the outage probability is notably decreased as the total number of satellites increases, which demonstrates the effectiveness of network scaling.
Besides, from Table \ref{tab:results}, we observe that the proposed TAEER outperforms the D-Merge method regardless of whether the LEO satellite network is dense or sparse. Notably, the advantage of TAEER is even more pronounced in sparse networks, highlighting the robustness of the proposed TAEER.

\subsubsection{Impact of the Receiver Telescope Size} 
Specifically, as shown in Fig. \ref{fig: ant simulation system}, we explore the relationship between LEO satellite receiver antenna size and the average energy consumption in each space topology. 
As the size of the antenna telescope increases, the receiver gain $G_R$ of LEO satellites significantly improves, resulting in stronger received signal power. Consequently, less transmission energy is required, which aligns with our energy consumption model \eqref{eq: weight frame}.
We can observe that under varying receiver antenna sizes, the three comparison algorithms maintain the same performance trends as the default settings. 

\subsubsection{Impact of the Receiver SNR Threshold}
We can conclude from \eqref{eq: outage prob} that the receiver SNR threshold plays a key role in the calculation of outage probability caused by pointing error. To evaluate the specific impact of ${\sf SNR}_{\text{th}}$ on the overall system performance, we summarize the simulation results in Fig. \ref{fig: SNR simulation system}.
As demonstrated in Fig. \ref{fig: SNR simulation system}, for both types of LEO satellite constellations, an increase in the threshold results in a substantial rise in the outage probability. This observation is consistent with our physical modeling of the probability distribution for outage.

\subsubsection{Visualization of Resulting Aggregation Paths}
Fig. \ref{fig: visualization} presents some sampling results for the aggregation paths, specifically the topology-aware routing outcomes for Time slot 0, Snapshot 25 under the Walker-Delta constellation, and Time slot 18, Snapshot 25 under the Walker-Star constellation generated by the proposed TAEER and D-Merge algorithm. 
These visualized aggregation paths demonstrate the correctness of our proposed TAEER and D-Merge algorithms.
In addition, the visualization results also show that, under the condition of correctly solving the \textit{DST} problem, the proposed TAEER algorithm is superior to D-Merge. This is because TAEER solves the problem from the perspective of the MSA and the overall topology, whereas D-Merge only merges the shortest paths.

\section{Conclusion}\label{sec: conclusion}
In this paper, we introduced a novel hierarchical learning and computing framework tailored for SGINs, which effectively supports collaborative training on widely distributed terrestrial IoT devices using LEO mega-constellations and GEO satellites.
We leveraged the predictability of satellite network topology by modeling the space network as a directed graph, where edge weights represent energy consumption during transmission. We revealed that minimizing the overall network transmission energy is equivalent to solving a \textit{DST} problem, which is NP-hard.
To address this, we proposed a TAEER algorithm. By first using the Dijkstra algorithm to find the minimum cost path to the root node and then identifying a MSA, we solve the \textit{DST} problem heuristically to create an effective aggregation routing path.
Furthermore, we extended our solution to account for unpredictability such as transmission outages. Simulation results confirmed its superior
performance and robustness compared to existing benchmarks.

\textit{Future Directions:} There are several open directions for extending this work.
First of all, detailed modeling of the energy harvesting and battery systems of the satellites can be considered to formulate a constrained \textit{DST} problem.
Second, this paper considers the space topology as quasi-static, meaning it varies between different time slots but remains constant within the same time slot, while the connection weights are time-varying within a time slot. 
A promising research direction is to model the space network as a time-varying graph for analysis \cite{han2022time}.
We also plan to explore allocating the proportion of transmitted information across time frames, which represents a data allocation problem under dynamic topology. 
Additionally, leveraging the correlation between topologies in adjacent time slots to reduce the complexity of topology construction is another important potential research direction.
Finally, we can explore the use of deep reinforcement learning and other low-complexity methods to solve the \textit{DST} problem for topology-aware routing.

	\apptocmd{\thebibliography}{\setlength{\itemsep}{-1pt}}{}{}	
	\bibliographystyle{IEEEtran}
	\bibliography{refs}

\end{document}